\documentclass{article}
\usepackage{amsmath, amsfonts, amsthm}
\usepackage{mathrsfs}
\usepackage{bbm}
\usepackage{natbib}
\usepackage{graphicx}
\usepackage{algorithm}
\usepackage[noend]{algorithmic}
\usepackage{float}
\usepackage{setspace} 
\usepackage{xcolor}
\usepackage{fullpage}
\usepackage{longtable}
\DeclareMathAlphabet{\mathpzc}{OT1}{pzc}{m}{it}

\title{Augmented Feature Boosting for Multicalibration}
\author{
Ira Globus-Harris\\
Cornell University\\
\texttt{img53@cornell.edu}
\and
Inbal Livni Navon\\
Ben-Gurion University of the Negev\\
\texttt{inballn@bgu.ac.il}
}
\date{}

\begin{document}
\newtheorem{theorem}{Theorem}[section]
\newtheorem{definition}[theorem]{Definition}
\newtheorem{lemma}[theorem]{Lemma}
\newtheorem{corollary}[theorem]{Corollary}
\newtheorem{example}[theorem]{Example}
\newtheorem{remark}{Remark}
\newtheorem{claim}[theorem]{Claim}
\newcommand{\eps}{\varepsilon}

\newcommand{\round}{\textup{Round}}
\newcommand{\discrete}{\textup{Discretize}}
\newcommand{\bZ}{\mathbb{Z}}
\newcommand{\barf}{\bar{f}}

\newcommand{\cA}{\mathcal{A}}
\newcommand{\cB}{\mathcal{B}}
\newcommand{\cC}{\mathcal{C}}
\newcommand{\cD}{\mathscr{D}}
\newcommand{\cE}{\mathcal{E}}
\newcommand{\cF}{\mathcal{F}}
\newcommand{\cG}{\mathcal{G}}
\newcommand{\cH}{\mathcal{H}}
\newcommand{\cN}{\mathcal{N}}
\newcommand{\cO}{\mathcal{O}}
\newcommand{\cR}{\mathcal{R}}
\newcommand{\cU}{\mathcal{U}}
\newcommand{\cS}{\mathcal{S}}
\newcommand{\cX}{\mathcal{X}}
\newcommand{\cY}{\mathcal{Y}}

\newcommand{\hcE}{\widehat{\mathcal{E}}}

\newcommand{\1}{\mathbbm{1}}
\newcommand{\E}{\mathbb{E}}
\newcommand{\R}{\mathbb{R}}
\newcommand{\clip}[0]{\textup{\texttt{clip}}}
\newcommand{\sgn}{\textup{Sign}}
\newcommand{\sign}{\textup{Sign}}
\newcommand{\init}{\textup{\texttt{init}}}
\newcommand{\finit}{f_{\textup{\texttt{init}}}}

\newcommand{\br}{\bar{r}}
\newcommand{\abs}[1]{\left| #1 \right|}

\newcommand{\err}{\text{err}}

\newcommand{\oracle}{\mathcal{O}_{\cH'}}
\newcommand{\clipa}{\cA^{\texttt{clip}}}
\newcommand{\bigquad}{\quad\quad\quad\quad\quad\quad\quad\quad}
\newcommand{\blah}{\textup{\texttt{blah}}}
\newcommand{\featureboost}{\textup{FeatureBoost}}
\newcommand{\treedepth}{\mathpzc{l}}

\newcommand\igh[1]{{[\color{magenta}{\textbf{Ira:}~#1}]}}
\newcommand\inbal[1]{{[\color{blue}{\textbf{Inbal:}~#1}]}}

\maketitle

\begin{abstract}
Multicalibration requires a predictor's residuals to be unbiased not only globally, but also after conditioning on the predictor's own level sets and reweighting by a rich class of test functions. Standard boosting approaches in the distributional setting achieve this by repeatedly discretizing the predictor's range then auditing and repairing the resulting level sets. One consequence is that in practice, the algorithm's guarantees are sensitive to this parametrization of the rounding parameter. A natural theoretical question, then, is how to do discretization-free boosting which avoids this rounding within the boosting process itself. Here, we analyze an alternative feature-augmentation boosting paradigm inspired by \cite{tax2026mcgrad}: at each round, a squared-loss oracle is called on hypotheses that receive the previous predictor's output as an additional feature, and only the final predictor is rounded to have a finite set of level sets to provide the multicalibration guarantee with respect to. We give a theoretical analysis of this procedure through the expressivity of the augmented hypothesis class, and show how the expressivity of this class yields a hierarchy of guarantees, including multiaccuracy, multicalibration, and the stronger notion of ``level-set" multicalibration.
\end{abstract}

\section{Introduction}

Empirical risk minimization---i.e., minimizing some standard (possibly proper) loss function within a set family of hypotheses $\cH$---is the primary cornerstone of algorithm design in machine learning. However, alternative objectives exist. For instance, in boosting, the goal is not to find the best model within $\cH$ for a particular learning task, but to use the models in $\cH$ as a collection of ``weak" learners with may then be ensembled or iteratively combined in some way to create a much stronger learner, despite not having the direct ability to do risk minimization over this richer class of ``boosted" models. 
Increasingly we want models that are used in the real world to have strong error guarantees with respect to something more refined than the average error of a predictor. In the context of regression models which map from contexts $\cX$ to real valued predictions in $[0,1]$, we may want their predictions to be interpretable in some meaningful way as \textit{probabilites}. At a minimum, this means predictions should be \textit{calibrated}: unbiased conditional on their own predictions. And, in order for this to be useful, this should hold not just on average on the full population, but for many different and possibly overlapping subpopulations. In other words, for any value $r$ in the range of the predictor $f$, the predictor must be unbiased not just on average but with respect to a weighted average generated by some test function $h$, for every $h$ in some class of tests $\cH$:
\[
\E_{(x,y)\sim\cD}\left[h(x)(f(x)-y) \vert f(x)=r \right] \approx 0.
\]

This form of guarantee, called \textit{multicalibration}, was originally introduced in the algorithmic fairness literature \cite{hebert2018multicalibration}---there, the goal was to generate a predictor which is calibrated not just overall but on many different (demographic) groups. Since then, multicalibration has emerged as a general algorithmic primitive with applications far beyond fairness: by cleverly designing the conditioning events that one wishes to enforce multicalibration on, one can get strong guarantees useful for downstream optimization \cite{gopalan2021omnipredictors,hu2023omnipredictors, downstreamfairness, ensemblingopt}, robustness under distribution shift \cite{kim2022universal}, information aggregation in collaborative or distributed settings \cite{agreement, collabpred}, computational complexity \cite{casacuberta2024complexity}, and online learning \cite{garg2024oracle}.

These applications have led to a growing interest in developing efficient algorithms for achieving multicalibration. The original algorithmic method provides a boosting-style approach to guaranteeing multicalibration via iteratively checking for violations of the guarantee and patching them \cite{hebert2018multicalibration}. This is relatively efficient: the number of iterations of boosting depends only on the desired calibration error, while the computational complexity is largely determined by the difficulty of finding violated calibration constraints. Subsequent work provides an alternative lens of multicalibration as a form of \textit{swap regret} with respect to the class of weak learners the boosting is over: a predictor has low multicalibration error with respect to a class $\cH$ if and only if its squared error cannot be improved by swapping out this collection of predictions by an alternative predictor $h$ in $\cH$. This suggests an alternative approach to building a multicalibrated predictor: in rounds, split predictions into a collection of (discretized) level sets, call a squared error oracle on each induced sub-distribution, and then swap out the current predictor on that level set for the new one, until the process no longer improves squared error \cite{lsboost}. 

These approaches all assume additional validation data is used to generate a multicalibrated predictor as a postprocessing step on top of the predictions of some initially trained model $f_0$. In other words, in order to be useful to applied practitioners, it must be the case that first using one subset of training data to train the initial predictor $f_0$ and then using the rest of it to multicalibrate that initial model gives a stronger multicalibration guarantee than simply using all of the data to train $f_0$. 

In practice, this does not seem to always be the case. Empirically, Hansen et al.~show that for some benchmark tasks and families of test functions, training models which heuristically achieve (standard) calibration---such as isotonic regression---give superior results for \textit{multi}calibration than splitting the dataset into an initial model training and posthoc multicalibration component \cite{postprocessingnecessary}. Thus, is going to the trouble of explicitly multicalibrating a predictor, and sacrificing some training data to do so, worth it? 

Part of the challenge of creating useful postprocessing for multicalibration is that the algorithms all rely on repeatedly discretizing predictions to a finite set of values at each boosting step---intuitively, this is necessary because multicalibration itself requires that you condition on an event which is in terms of the value of your own prediction. The boosting step treats these separate ``buckets" of discretized predictions independently. This in turn means that they cannot include too few samples, or else generalization will fail. Hence, the sample complexity of the algorithm is sensitive to how this rounding hyperparameter is chosen. And, since this discretization is applied internally to the boosting process itself at every round, folk theory tells us misspecification in the hyperparameters compound over time. This leads us to a natural question: 

\begin{center}
\textit{Are there discretization-free methods to build multicalibrated predictors with better sample complexity?}
\end{center}

Recent work gives some indication of an affirmative: theoretical work shows that for some sufficiently rich hypothesis classes, standard risk minimization also achieves multicalibration \cite{blasiok2023loss}. However, this work is not constructive. While, for instance, \citep{blasiok2023loss} show that perfect loss minimization on neural nets of depth $k$ implies multicalibration with respect to functions representable by nets of some smaller size $k'$, it does not provide a construction of \textit{which} $k'$ this might be. 

On the empirical side, Tax et al.~\cite{tax2026mcgrad} propose a fundamentally different boosting paradigm based on \emph{feature augmentation}: at every boosting iteration, the predictor from the previous round is added as an additional feature for the next weak learner, without discretization between rounds. They demonstrated empirically that this approach produces multicalibrated predictors. In follow-up work, they establish theoretical guarantees under substantially restricted assumptions \cite{haimovich2026convergence}. 

In this work, we provide the first general theoretical analysis of this style of ``feature-augmented" boosting for multicalibration. We characterize minimal requirements on the augmented hypothesis class needed to achieve multiaccuracy, multicalibration, and a stronger notion of level-set multicalibration---while only requiring one single discretization step on the final predictor, rather than throughout the boosting process. Since this can be done entirely post-hoc, it is thus a remarkably flexible paradigm, where hyperparameter tuning can be completed after any boosting occurs. We further show that for standard calibration guarantees, this final discretization step is strictly necessary to guarantee small calibration error.

Finally, our analysis reveals a new trade-off between the expressive power of the augmented hypothesis class and the complexity of learning. Richer augmentation allows stronger guarantees and improves the dependence between the boosting error and the resulting calibration error, but simultaneously increases the complexity of the underlying learning problem. We formalize this trade-off by deriving sample complexity bounds for the augmented hypothesis classes. 

\subsection{Our results}
In this work, we analyze a boosting algorithm for multicalibration based on \emph{feature augmentation}, inspired by Tax et al.~\cite{tax2026mcgrad, haimovich2026convergence}. Starting from an initial predictor $f_0$, each boosting iteration learns a hypothesis $a\in\cA$ that receives both the original features $x$ and the current prediction $f_t(x)$ as input and defines the next predictor by
\[
f_{t+1}(x)=a(x,f_t(x)).
\]
The algorithm continues updating the predictor as long as the loss decreases by some additive slack, and terminates once no hypothesis in the augmented class $\cA$ can achieve such an improvement.

Our main contribution is a characterization of the expressive power required from the augmented hypothesis class. We show that richer augmented classes yield progressively stronger calibration guarantees, ranging from multiaccuracy to multicalibration and a stronger notion of multicalibration called ``level-set" multicalibration with respect to a benchmark class of hypotheses $\cB$.

\begin{itemize}
    \item If $\cA$ can express updates of the form $
    a(x,u)=u+\eta b(x),$
    for every bounded $b\in\cB$ and $\eta\in\R$, then Algorithm~\ref{algo:boosting} outputs a multiaccurate predictor with respect to $\cB$.

    \item If $\cA$ can additionally localize updates to prediction intervals, namely that it can express $a(x,u)=u+\eta b(x)\1(u\in S_i),$
    for every interval $S_i$ in a partition of $[0,1]$, then the algorithm outputs a predictor whose discretized version is multicalibrated with respect to $\cB$.

    We further show that this richer augmentation class leads to an improved dependence between the boosting accuracy parameter and the resulting calibration error.
    \item Finally, if $\cA$ can further localize updates according to both the prediction and the value of the auditing function, namely, $a(x,u)=u+\eta b(x)\1(u\in S_i,\,b(x)\in S_j)$,
    then the algorithm outputs a predictor whose discretized version is level-set multicalibrated with respect to $\cB$.\footnote{We formalize this notion of multicalibration in Section \ref{sec:level-set-mc}; at a high level, it requires swap regret not just on the predictor $f$'s level sets but jointly on $f$ and the test function $b$'s level sets.}
\end{itemize}

Unlike the standard algorithm for multicalibration, our approach performs a single discretization step only at the end of the algorithm. We prove that this final discretization step is, in general, necessary: without it, the resulting predictor may have arbitrarily large continuous calibration error. This phenomenon arises because continuous calibration conditions on events of the form $f(x)=r$, making it highly sensitive to infinitesimal perturbations of the predictions.

Finally, we give sample complexity guarantees for augmented feature boosting in terms of the error and multicalibration guarantees. The general guarantee is achieved by bounding the complexity $d_T$ of hypotheses generated after $T$ rounds of boosting, giving us the standard uniform convergence guarantee for 
\[
n \ge \tilde{O}\left(\frac{\log(1/\delta) + d_T}{\eps^2}\right).
\]
Here, it is important to note that this \textit{does not necessarily} beat the standard boosting-to-multicalibrate sample complexity bounds of around $O(1/\eps^4)$--it does not make the cost of discretization disappear. Rather, this cost shifts from a per-round penalty generated by explicitly bucketing predictions into level sets into the expressivity required of the weak learner itself, which in turn makes $d_T$ itself a function of $\eps$.

We then bound the pseudodimension of the boosted predictor in terms of its arithmetic expressivity---if the base weak learners have bounded description size $S'$ in terms of the number of real values they predicate on and the number of arithmetic operations they perform, then $d_T$ is bounded above by a function of $S$. We then instantiate this specifically for weak learners which are axis-aligned decision trees, and provide an example elucidating why such arithmetic expressivity bounds are necessary for generalization to hold.

\subsection{Related works}
Since its introduction, multicalibration has evolved from a fairness notion into a general framework for obtaining reliable probabilistic predictions. This broad applicability has motivated a growing body of work on both understanding the theoretical properties of multicalibration and developing more efficient algorithms for achieving it. 
 
Standard multicalibration approaches boost a predictor by efficiently calling some form of auditing oracle in order to test (and then repair) for violations of the calibration guarantees. This process repeats for some number of rounds, until all the tests pass, while still guaranteeing that the total number of tests the oracle is invoked for is small in order to to reuse samples efficiently. Subsequent works extended this framework to richer collections of tests, including omnipredictors~\cite{gopalan2021omnipredictors,hu2023omnipredictors}, and proposed weaker notions such as calibrated multiaccuracy~\cite{gopalan2022loss}, which can achieve improved computational or statistical complexity while still relying on iterative post-processing and discretization.

Unlike the approaches discussed above, which get sample complexity at least $\tilde{O}(\eps^{-4})$, the sample complexity of multicalibration is known to be tight at $\tilde{O}(\eps^{-3})$. This is achieved not by a standard boosting algorithm in the distributional/batch setting, but by treating the problem in the online setting where labels are adversarially generated, and then reducing back to the distributional setting via an online-to-batch reduction \cite{mcSampleComplexity}. This results in a randomized predictor. In recent work, Noarav and Roth show that this randomization is not necessary, and achieve a deterministic predictor with the same optimal rates \cite{deterministicMC}. However, their approach still requires the online-to-batch reduction. In contrast, our goal is not to optimize the
minimax dependence on $\eps$, but to analyze a boosting procedure whose updates are obtained via risk minimization and boosting.

Several works consider multicalibration in this agnostic learning framework. In particular, the omnipredictors framework \cite{gopalan2021omnipredictors} reformulates multicalibration through the lens of hypothesis classes, connecting it to weak agnostic learning and showing that multicalibration can be achieved using an oracle for weak learning over the underlying hypothesis class. In subsequent work, Globus-Harris et al.~consider the relationship between agnostic boosting and multicalibration in the context of regression \cite{lsboost}. These viewpoints highlight that the complexity of achieving multicalibration is governed by the complexity of the hypothesis class used to identify calibration violations.

Several works have therefore sought computationally weaker alternatives to full multicalibration with improved sample complexity. Examples include multiaccuracy \cite{kim2019multiaccuracy}, calibrated multiaccuracy \cite{gopalan2022loss}, and low-degree multicalibration \cite{gopalan2022low}, each sacrificing some expressive power in exchange for improved computational or statistical efficiency. However, these weaker notions are insufficient for several applications. For example, calibrated multiaccuracy guarantees omniprediction only for a restricted family of downstream optimization problems, whereas full multicalibration and level-set multicalibration provide guarantees for increasingly richer classes of downstream decision-making tasks \cite{gopalan2022loss,hu2023omnipredictors}.

Another direction is to develop simple alternative post-processing algorithms. Jin et al.~proposed a discretization-free post-processing procedure based on adding a sum of depth-$2$ decision trees to an existing predictor \cite{jin2026discretization}. Their algorithm is again based on boosting for loss minimization, and they show that after a final discretization step the resulting predictor has small multicalibration error.

The work most closely related to ours is the boosting framework introduced by Tax et al.~\cite{tax2026mcgrad}. Rather than explicitly identifying calibration violations, they repeatedly augment the input of the weak learner with the prediction of the current model and optimize the prediction loss. They provide empirical evidence that this procedure produces multicalibrated predictors and establish theoretical guarantees under additional assumptions and for a non-standard notion of multicalibration. In follow-up work, Haimovich et al. further analyzed the same algorithmic approach, proving an average notion of multicalibration in which the output predictor is calibrated in expectation over a randomly selected hypothesis from a finite hypothesis class ~\cite{haimovich2026convergence}. We compare our results to theirs more thoroughly in Appendix \ref{sec:differences}; at a high level, their theoretical guarantees are on average with respect to the group of test functions, while ours is a worst-case guarantee over all of these tests.

\section{Preliminaries}
We consider a learning task with feature space $\cX$ and real-valued one-dimensional labels $\cY\subseteq [0,1]$ over distribution $\cD \in \Delta(\cX \times \cY)$. We write $D \sim \cD^n$ to denote a dataset of $n$ labeled examples drawn independently and identically from $\cD$, and will use $[n]$ to denote $\{1,2,\ldots,n\}$. We will be interested in the squared error of the model, both over the entire distribution and over specific induced conditional distributions:

\begin{definition}
    Let $f:\cX\rightarrow[0,1]$, the squared error of $f$ under distribution $\cD\subseteq \cX\times [0,1]$ is
    \begin{align*}
        L_{\cD}(f)= \E_{(x,y)\sim\cD}\left[(f(x)-y)^2\right].
    \end{align*}
\end{definition}

When the expectations are taken over an empirical dataset $D \sim \cD^n$, we abuse notation and let $D$ denote a distribution which is uniform over the sample. 

Throughout this work, we develop algorithmic methods to build \textit{calibrated} predictors, which are unbiased conditional on their own predictions. 

\begin{definition}[Expected calibration error]
    The calibration error of a function $f:\cX\rightarrow R$ with a finite codomain $R\subset[0,1]$ under distribution $\cD$ is
    \begin{align*}
        ECE(f)=\sum_{r\in R}\Pr_{(x,y)\sim \cD}[f(x)=r]\abs{\E_{(x,y)\sim\cD}\left[\left. f(x)-y \right|f(x)=r\right]}.
    \end{align*}
\end{definition}

We will be interested in stronger guarantees, where we condition not only on the prediction but on these predictions constrained to any weighted subsets of the distribution defined by a class of models $\cB$. Or, equivalently, we will require a form of swap regret with $\cB$: that no model in $\cB$ has correlation with our predictor $f$'s residuals, conditioned on any of its own predictions. 

\begin{definition}[Multicalibration]
\label{def:mc}
    A function $f:\cX\rightarrow R$ with a finite $R\subset[0,1]$ is multicalibrated with respect to a set of functions $\cB$
    under distribution $\cD$  with error $\eps$ if for all $b\in\cB$
    \begin{align*}
        K_1(\cD,f) = \sum_{r\in R}\Pr_{(x,y)\sim \cD}[f(x)=r]\abs{\E_{(x,y)\sim\cD}\left[\left. b(x)(f(x)-y) \right|f(x)=r\right]}\leq\eps.
    \end{align*}
\end{definition}

We will develop a boosting algorithm which, unlike standard approaches, is relatively discretization-free, in the sense that at every round of boosting, the predictor is not replaced with one of a set number of predictors or rounded to a discrete range. However, we wish to achieve the standard notion of $K_1(\cD, f)$ multicalibration, which conditions on a level set $r$ of the predictions, in a way that generalizes out-of-sample. Hence, our final multicalibration guarantee will still be with respect to a final $\textit{discretized}$ predictor, which we define as follows.

\begin{definition}\label{def:k-bins}
    The partition of $[0,1]$ into $K$ equal-sized bins is $\cS = \{[0,1/K),[1/K,2/K),\ldots[(K-1)/K,1] \}$.
\end{definition}

\begin{definition}[Discretization]
\label{def:discretization}
Let $f:\cX\rightarrow [0,1]$ be a function and $\cS_K$ be a partition of $[0,1]$ into $K$ equally-sized bins. Let $\cH$ be the family of all functions mapping to $\R$. We will then let $\discrete: \cF \rightarrow \cF$ return a function $\barf = \discrete({f;\cS_K})$ which maps $f(x)$ to $\E[f(x) \vert x \in S]$, where $S$ is the unique bin in $\cS_K$ such that $f(x) \in S$.
\end{definition}

Along the way to develop algorithms which achieve multicalibration, we will also show that our approach gives us the weaker notion of multi\textit{accuracy}, which guarantees that no model in the comparator class $\cB$ has overall expected correlation with $f$'s residuals.

\begin{definition}[Multiaccuracy]
    A function $f:\cX\rightarrow [0,1]$ is multiaccurate with respect to a set of functions $\cB$
    under distribution $\cD$  with error $\eps$ if for all $b\in\cB$
    \begin{align*}
    \abs{\E_{(x,y)\sim\cD}\left[ b(x)(f(x)-y) \right]}\leq\eps.
    \end{align*}
\end{definition}

\section{Augmented Feature Boosting}
For our boosting process, we will be optimizing over a hypothesis class that \textit{augments} its feature space with the predictions that the previous round of boosting output. We formalize this ``augmented" hypothesis class of models as follows. Within each round, the boosting process will then call a squared error oracle on the induced mapping on class $\cA$ which uses the predictions from the previous round as an additional feature.

\begin{definition}[Augmented Hypothesis Class $\cA$]
Let $\cA$ be a class of mappings $a:\cX \times \R \rightarrow \R$. Fixing some model $f: \cX \rightarrow \R$, we will refer to the mappings of $\cA$ induced by augmenting the features $\cX$ with the predictor $f$ as
\[
\cA[f] = \{x \mapsto a(x,f(x)) : a \in \cA\}.
\]
\end{definition}

\begin{definition}[Approximate squared-error oracle for augmented hypotheses $\cO_{\cA}$]
\label{def:oracle}
Let $\cA$ be an augmented hypothesis class of models, and fix some model $f: \cX \rightarrow \R$. Let $\tau \in [0,1]$ be a tolerance parameter. We call $\cO_{\cA}$ a $\tau$-approximate oracle for $\cA$ with respect to $f$ if for every distribution $\cD \in \Delta(\cX \times \R),$ it returns a function $a \in \cA$ satisfying 
\[
\E_{(x,y)\sim \cD}[(a(x,f(x))-y)^2] \le \min_{a' \in \cA} \E_{(x,y)\sim \cD}[(a'(x,f(x))-y)^2] + \tau.
\]
We denote such an oracle call by 
\[
a \leftarrow \cO_{\cA}(\cD,f,\tau).
\]
\end{definition}

In order to satisfy our sample complexity guarantees in Section \ref{sec:generalization}, we will need to require that the class of models that we boost over is in fact a clipped version of the hypothesis space, which we formalize as follows. 

\begin{definition}[Clipping function \clip] For any $u,v \in \R$ with $u<v$, let $\clip_{[u,v]}: \R \rightarrow [u,v]$ be the function which clips values to the range $[u,v]$. I.e., 
\begin{equation*}
\clip_{[u,v]}(x) = 
\begin{cases}
u & \text{if } x < u, \\
v & \text{if } x > v, \\
x & \text{else.}
\end{cases}
\end{equation*}
Throughout the paper, we will use $\clip(x)$ as a shorthand for $\clip_{[0,1]}(x)$.
\end{definition}

\begin{definition}[Clipped Hypothesis Class $\clipa$]
    Let $\cA$ be an augmented hypothesis class of models $a: \cX \times \R \rightarrow \R$. We will write $\clipa$ to be $\cA$ clipped to the range $[0,1]$. I.e.,
    \[
    \clipa = \{\clip(a(x,u)) : a \in \cA\}.
    \]
\end{definition}

We now have the necessary machinery to formalize the boosting process for augmented feature boosting in Algorithm \ref{algo:boosting} and demonstrate that the process halts.

\begin{algorithm}[H]
\begin{spacing}{1.15}
\begin{algorithmic}
\STATE {\bf Input:} Augmented hypothesis classes $\cA $ of models $a: \cX \times \R \rightarrow \R$, distribution $\cD \in \Delta (\cX  \times \cY)$, initial model $\finit: \cX \rightarrow \R$, hyperparameters $\alpha, \tau$.
\STATE Let $f_0 = \clip(\finit)$
\STATE Let $\err_{-1}=\infty$ and let $\err_{0} = L_{\cD}(f_0)$
\STATE Let $t = 0$
\WHILE{$\err_{t-1}-\err_{t} \ge \alpha-\tau$}
    \STATE Let $t = t + 1$
    \STATE  Let $a_t \in \cO_{\clipa}(\cD, f_{t-1},\tau)$ \quad\quad\quad\quad Note the oracle is called on $\clipa$, the clipped version of $\cA$. 
    \STATE Let $f_t(x) = a_t(x, f_{t-1}(x))$
    \STATE Let $\err_t = L_{\cD}(f_t)$
\ENDWHILE 
\STATE {\bf Output: $f_{t-1}$} 
\end{algorithmic}
\end{spacing}
\caption{$\featureboost(\cA,\cD,\finit,\alpha, \tau)$; Boosting by Feature Augmentation}
\label{algo:boosting}
\end{algorithm}

\begin{theorem}[$\featureboost(\cA,\cD,\finit,\alpha, \tau)$ halts] Fix a distribution $\cD \in \Delta (\cX \times \cY)$ and model $\finit: \cX \rightarrow \R$. Let $\cA$ be an augmented hypothesis class, and let $\alpha \in (0,1), \tau \in [0,\alpha)$. Then, $\featureboost(\cA,\cD,\finit,\alpha,\tau)$ (Algorithm \ref{algo:boosting}) halts after at most $T = \lfloor L_{\cD}(\finit)/(\alpha-\tau) \rfloor \le 1/(\alpha-\tau)$ accepted updates, and $\lfloor L_{\cD}(\finit)/(\alpha-\tau) \rfloor+1 \le 1/(\alpha-\tau) + 1$ oracle calls.
\end{theorem}

\begin{proof}
Consider an instantiation of $\featureboost(\cA,\cD,\finit,\alpha,\tau)$ over $T \in \bZ$ rounds. Note that for any $t \in \{0,1,\ldots,T\}$, all predictions output by $f_t$ must lie in $[0,1]$, since $\finit$ is explicitly clipped to this range and the squared error oracle $\cO$ is invoked with respect the a clipped version of the augmented class $\cA$. Hence, squared error $L_{\cD}(f_t)$ is always bounded within $[0,1]$, and particularly $\err_0 = L_{\cD}(\finit) \le 1$. An update is accepted if the error drops by at least $\alpha-\tau$. Thus, since squared error is non-negative, the algorithm can accept at most $\lfloor L_{\cD}(\finit)/(\alpha-\tau) \rfloor$ updates, where each update includes one oracle call. In the subsequent round, the drop in error resulting from the one additional oracle call must be less than $\alpha-\tau$, since otherwise there would be an additional update, contradicting the fact that squared error must be nonnegative. 
\end{proof}

\section{Analysis for Additive Factorized Hypothesis Classes}
In general, what kinds of guarantees should the above process give us? The challenge in formalizing a multicalibration-style guarantee is that is not clear in full generality how to reason about the boosting step, as the interaction the learner has between the feature set and the previous round's predictions could be complex. All hope, however, is not lost---note that for many natural forms of hypotheses, the optimization process may be \textit{separable} in how they interact with the features and the previous round's predictions. We formalize this by considering hypothesis classes which are \textit{factorized}, i.e. can be written in the form $b(x)c(f(x))$. This, however, is insufficient to give calibration-style guarantees; it will be crucial that our augmented hypothesis class of predictors can also be \textit{shifted} by constant values in order to demonstrate that the risk minimization process precludes the ability to ``patch" any bias in the predictor by subtracting off that bias. We formalize this as follows.

We remark that for a strong type of multicalibration the additive factorized class do not suffice, and in Section \ref{sec:level-set-mc} we consider a different class of update functions.

\begin{definition}[Additive factorized augmented class $\cA$]
Let $\cA$ be a augmented class of models $a: \cX \times \R \rightarrow \R$. We will say that $\cA$ is an additive factorized augmented hypothesis class with respect to hypothesis classes $\cB$ of models $b: \cX \rightarrow \R$ and $\cC$ of models $c: \R \rightarrow \R$ if for all $b \in \cB,c \in \cC$ and $\eta\in \R$ there exists $a\in\cA$ such that $a(x,u) = u + \eta b(x) \cdot c(u)$. 
 \end{definition}

 We now observe that for such classes, we can bound the correlation between the final predictor $f$ output by $\featureboost$ and the class $\cA$. 
 
\begin{theorem}\label{thm:augmented-general}
Let $\cA$ be an additive factorized hypothesis class with respect to classes $\cB,\cC$ . Fix $\finit: \cX \rightarrow \R$. Let $f \leftarrow FeatureBoost(\cA, \cD, \finit, \alpha, \tau)$ for some distribution $\cD$ and hyperparameters $\alpha, \tau \in (0,1]$. Then $\forall M\in\R, b\in\cB$ with $\E[b^2(x)]\leq M$  and for all $c\in\cC$ with $\forall u, |c(u)|\leq 1$ we have
 \begin{align}\label{eq:general-bound}
     \abs{ \E_{(x,y)\sim\cD}[c(f(x))b(x)(y-f(x))]} \le \sqrt{\alpha M}
 \end{align}
\end{theorem}

\begin{proof}
    Let $f$ be the output of the algorithm, and assume towards a contradiction that there exists $b\in\cB$ and  $c\in\cC$ such that  $\E[c^2(u)b^2(x)]\leq M$ for all $x,u\in[0,1]$ and  $|\E[c(u)b(x)(f(x)-y)]| > \sqrt{\alpha M}$.

    Let $\sigma = \text{Sign}(\E[c(u)b(x)(f(x)-y)])$, and let $f'(x) =f(x) - \sigma\eta b(x)c(f(x))$, where $\eta $ is decided later.
We show that $f'$ reduces squared the error by more than $\alpha$, contradicting the algorithm's grantee:
\begin{align*}
    L_\cD(f')-L_\cD(f) \leq& \E[(f'(x)-y)^2] - \E[(f(x)-y)^2]\\=& \E\left[\left(f(x)-\eta \sigma b(x)c(f(x))-y\right)^2\right] 
    - \E[(f(x)-y)^2]\\=&
    \E\left[-2\sigma \eta b(x)c(f(x))(f(x)-y)+\eta^2\sigma^2b^2(x)c^2(f(x))\right] \\\leq&
    \eta^2\E\left[c^2(f(x))b^2(x)\right] -2\eta\sigma  \E\left[ b(x)c(f(x))(f(x)-y)\right]  \\ \leq &
    \eta^2 M -2\eta\sqrt{\alpha M}.
\end{align*}
We choose $\eta = \sqrt{\frac{\alpha}{M}}$ and get that
$ \eta^2 M -2\eta\sqrt{\alpha M} = \alpha - 2\alpha = -\alpha$.
Let $a\in\cA^{\clip}$ be the clipped $f'$, i.e. for every $x,u$ we have $a(x,u)=\clip(f'(x,u))$. Since $y\in\{0,1\}$, the expected loss of $a$ is at most the expected loss of $f'$ and we have $L_\cD(a)\leq L_\cD(f') \leq L_\cD(f)-\alpha$.
The oracle $\cO_{\cA^\clip}$ should return a hypothesis $h\in\cA^\clip$ such that $L_\cD(h)\leq L_\cD(a)+\tau\leq L_\cD(f)-\alpha+\tau$. In this case, the algorithm would not have stopped and output $f$, contradicting our assumption that the algorithms' output $f$ violates (\ref{eq:general-bound}).

\end{proof}

We show in the next section that for different types of bounded function classes $\cC$, the above theorem implies that the model $f$ which the boosting algorithm outputs is a multiaccurate predictor. We then show that a \textit{rounded} version of the final predictor, $\barf = \discrete(f;\cS_K)$, which discretized the predictions to a discrete grid $\{0,1/K,\ldots,1\}$, is multicalibrated. Note this rounding only happens on the final model output, which is a substantial difference from existing algorithms, which require the discretization at every boosting step--there, we only require that the models be clipped to some finite range. We then show that by increasing the expressivity of the class $\cC$, we get stronger forms of multicalibration.

\section{Expressivity of factor class $\cC$ and Calibration Guarantees}\label{sec:express}
We stay in the setting where our augmented hypothesis class $\cA$ is an additive factorized hypothesis class with respect to $\cB$ and $\cC$. Thus, as in the previous section, we can write the final model $f$ which the boosting process generates after $T$ rounds as 
\[
f = a(x,f_{T-1}(x)) = f_{T-1}(x) + \eta b(x) \cdot c(f_{T-1}(x)) \quad\quad  \text{for }b \in \cB, c \in \cC.
\]
Our final goal will be to achieve multicalibration guarantees with respect to $\cB$. Note that this guarantee is with respect to a function class that is a mapping from the features $\cX$ to $\cY$ and thus is not itself defined in any complex or recursive way on the boosting process itself. We will see that by assuming various closure properties on the expressivity of $\cC$, which encodes how the previous round's predictions are used, we can achieve a variety of (multi)calibration guarantees---as long as we round the predictions that $\featureboost$ returns to some final discretization. 

\subsection{When $\cC$ contains constant functions, we 
achieve multiaccuracy with respect to $\cB$}\label{sec:ma}
First, we note that with only a minimal assumption---closure with respect to constant functions---we achieve multi\textit{accuracy} with respect to the class $\cB$. For this to hold, we also need not round the predictions, since the guarantee is not conditional on our predictor's level sets but rather is in expectation overall. 
\begin{lemma}
    Let $\cC$ be a function class containing the constant function $c(u)=1 \ \forall u\in \R$. Let $\cA$ be an additive augmented class with respect to hypothesis classes $\cB$ and $\cC$. Let $\cB_M\subset\cB$ be the set of functions in $\cB$  such that $\E[b^2(x)]\leq M$ for $M \ge 0$. Let $f \leftarrow FeatureBoost(\cA, \cD, f_0, \alpha, \tau)$ for some distribution $\cD$ and hyperparameters $\alpha, \tau \in (0,1]$. Then  $f$ is $\sqrt{\alpha M}$-multiaccurate with respect to $\cB_M$.
    
\end{lemma}
\begin{proof}
    From Theorem \ref{thm:augmented-general} with $c=1$, we have that \[
    | \E[b(x)(y-f(x))]| \le \sqrt{\alpha M}
    \]
    for all $b\in\cB$ such that $\E[b^2(x)]\leq M$. This directly implies multiaccuracy with respect to $\cB_M$.
\end{proof}

\subsection{When $\cC$ contains intervals,we get multicalibration with respect to $\cB_M$} 
\label{sec:interval-mc}

We now show that if we round the predictor output by $\featureboost$ into some $r$ bins following the discretization process defined in \ref{def:discretization} we achieve canonical $K_1$ multicalibration. Here, this multicalibration guarantee is not with respect to $\cB$--for arbitrary real-valued function classes that are closed under scaling, this is impossible to achieve for approximate multicalibration. Instead, we give the guarantee for a bounded subset $\cB_M$, of $\cB$. We later will show in section \ref{sec:rounding} that this post-hoc rounding process is necessary.

\begin{lemma}\label{lem:mc}
   Let $M\ge0$ and let $\cB_M\subset\cB$ be the set of all $b\in\cB$ such that $\E[b^2(x)]\leq M$. Let $\cS$ be a partition of $[0,1]$ into $K$ equal-sized bins. Say $\cC$ is a set of functions containing indicators for all $S\in \cS$, and let $\cA$ be an additive factorized augmented class with respect to $\cB$ and $\cC$. Let $f \leftarrow FeatureBoost(\cA, \cD, f_0, \alpha, \tau)$ for some distribution $\cD$ and hyperparameters $\alpha, \tau \in (0,1]$. Then  $ \bar f = \discrete(f;\cS)$, the discretization of $f$ to $\cS$ is $r\sqrt{\alpha M}+\frac{\sqrt{M}}{K}$-multicalibrated with respect to $\cB_M$.

\end{lemma}
\begin{proof}
    Let $f$ be the function satisfying the conditions of the lemma and $\bar f$ be the discretization of $f$ to $\cS$, as in Definition \ref{def:discretization}. Fix any function $b\in\cB_M$. 
    In this notation, calibration error of $\bar f$ on $b$ is given by:
    \begin{align*}
        \sum_{S\in\cS}  \abs{\E[b(x)(y-v_S)\1(\bar f(x)= v_S)]} =& \sum_{S\in\cS}  \abs{\E[b(x)(y-v_S)\1( f(x)\in S)]
        },
    \end{align*}
    where for each $S$, $v_S=\E[f(x)|f(x)\in S]$.
    For every $S\in\cS$ we have that
    \begin{align*}
         \abs{\E[b(x)(y-v_S)\1 (f(x)\in S)]} =& \abs{\E[b(x)(y-f(x)+f(x)-v_S)\1 (f(x)\in S)]} \\\leq&
         \abs{\E[b(x)(y-f(x))\1 (f(x)\in S)]} + \abs{\E[b(x)(f(x)-v_S)\1 (f(x)\in S)]}.
    \end{align*}
    From Theorem \ref{thm:augmented-general}  with $c(u) = \1(u\in S)\in \cC$, we get that $\abs{\E[b(x)(y-f(x))\1 (f(x)\in S)]}\leq \sqrt{\alpha M}$.
    For the second element in the sum, we note that $\abs{f(x)-v_S}\leq 1/K$, as this is the width of each bin in $\cS$. Therefore, we find that
    \begin{align*}
        \abs{\E[b(x)(f(x)-v_S)\1 (f(x)\in S)]}\leq \frac{1}{K} \abs{\E[\abs{b(X)}\1 (f(x)\in S)]}.
    \end{align*}

    Combining everything together:
    \begin{align*}
        \sum_{S\in\cS}  \E[b(x)(y-v_S)\1(\bar f(x)= v_S)] \leq& \sum_{S\in\cS}\sqrt{\alpha M} + \frac{1}{K}\E[\abs{b(X)}\1 (f(x)\in S)]\\
        \leq & K\sqrt{\alpha M} + \frac{1}{K}\E[\abs{b(X)}] \\ \leq&
        K\sqrt{\alpha M}+\frac{\sqrt{M}}{K}.
    \end{align*}
    Where we bound $\E[\abs{b(X)}]\leq \sqrt{\E[\abs{b^2(X)}]}=\sqrt{M}$.
    
\end{proof}

\begin{remark}
We note that when running the boosting algorithm, we should choose the class $\cC$ to contain discretization for small enough bins, then choose $\alpha$ to ensure that the total calibration error is small. If we want to be $\eps$-multicalibrated with respect to an $M$-bounded class, we can choose $K=\left\lceil \frac{\sqrt{M}}{2\eps} \right\rceil$ and choose $\alpha = \frac{\eps^2}{4M K^2} \leq \frac{\eps^4}{2 M^2}$. In the next section we show that a richer class of functions can result in better dependency.  
\end{remark}

\subsection{When $\cC$ contains weighted sums of indicators, the multicalibration error is quadratic}\label{sec:signed-intervals-mc}

In the previous section, our error guarantee was quartic in its multicalibration error. Here, we show that be using a richer set of function as $\cC$ we can avoid the $K$ factor to the error parameter, and reduce the dependency on $\eps$ from quartic to quadratic. Thus, for an instantiation of the algorithm with stopping parameter $\alpha$ we will achieve stronger error guarantees than in Section \ref{sec:interval-mc}. First, we build some additional notation so that we can consider models in $\cC$ which are not just effectively the level sets of a predictor but instead contain \textit{weighted sums} of these buckets. 

\begin{definition}\label{def:bins-signs}
    Let $\cS$ be a set of disjoint bins in $[0,1]$, denoted by $S_1,\ldots,S_K$. For every $\sigma\in\{-1,1\}^{K}$ we define $g_\sigma:[0,1]\rightarrow [-1,1]$  by: $g_\sigma(u)=\sum_j\1(u\in S_j)\sigma_j$
    Denote the set of signed bins by:
    \[\Psi = \left\{ g_\sigma |\sigma\in\{-1,1\}^K\right\}.  \]
\end{definition}
We note that since  $\cS$ contains disjoint intervals, $|c(u)|\leq 1$ for every $u$, as required in the conditions of Theorem \ref{thm:augmented-general}.
\begin{lemma}
\label{lemma:boosting-mc-value}
   Let $\cS$ be a partition of $[0,1]$ into $K$ equal-sized bins and let $M\ge0$. Let $\cC$ be a set of functions containing the set $\Psi$ of signed bins for $\cS$, as in Definition \ref{def:bins-signs}. Let $\cB_M$ be the set of all $b\in\cB$ such that $\E[b^2(x)]\leq M$. Let $f \leftarrow FeatureBoost(\cA, \cD, f_0, \alpha, \tau)$ for some distribution $\cD$ and hyperparameters $\alpha, \tau \in (0,1]$. Then $\bar f$, the discretization of $f$ to $\cS$, is $\sqrt{\alpha M}+ \frac{\sqrt{M}}{K}$-multicalibrated with respect to $\cB_M$.
\end{lemma}
\begin{proof}
    Fix any function $b\in\cB_M$. 
    In this notation, calibration error of $\bar f$ on $b$ is given by
    \begin{align}\label{eq:calib-error}
        \sum_{S\in\cS}  \E[b(x)(y-v_S)\1(\bar f(x)= v_S)] =& \sum_{S\in\cS}  |\E[b(x)(y-v_S)\1( f(x)\in S)]|.
    \end{align}

\noindent From the triangle inequality, for every $S\in\cS$ we have that
\begin{align*}
    \abs{\E[b(x)(y-v_S)\1( f(x)\in S)]}\leq \abs{\E[b(x)(y-f(x))\1( f(x)\in S)]}+\abs{\E[b(x)(f(x)-v_S)\1( f(x)\in S)]}.
\end{align*}
Therefore, we can bound the calibration error in Equation (\ref{eq:calib-error}) by:
\begin{align}\label{eq:calib-error-2}
        \sum_{S\in\cS}  \abs{\E[b(x)(y-v_S)\1(\bar f(x)= v_S)]} \leq& \sum_{S\in\cS}\abs{\E[b(x)(y-f(x))\1( f(x)\in S)]}\nonumber\\&+\sum_{S\in\cS}\abs{\E[b(x)(f(x)-v_S)\1( f(x)\in S)]}
    \end{align}

\noindent For every bin $S\in\cS$, let
\begin{align*}
\sigma_S &= \sign\left(\E[b(x)(y-f(x))\1(f(x)\in S)]\right).
\end{align*}
Then, the following function $g_\sigma$ is in $\cC$:
\[
g_\sigma(u)=\sum_{S\in\cS}\sigma_S\1(u\in S).
\]
Using $g_\sigma$ we can write
\begin{align*}
    \sum_{S\in\cS}\abs{\E[b(x)(y-f(x))\1(f(x)\in S)]}= \E[g_\sigma(f(x))b(x)(y-f(x))]\leq \sqrt{\alpha M},
\end{align*}
where the inequality is by applying Theorem \ref{thm:augmented-general} on the functions $g_\sigma\in\cC,b\in\cB$. Then, using the above bound in calibration error in Equation (\ref{eq:calib-error-2}) we find that
\begin{align}\label{eq:calib-error3}
        \sum_{S\in\cS}  \E[b(x)(y-v_S)\1(\bar f(x)= v_S)] \leq& \sqrt{\alpha M} + \sum_{S\in\cS}\abs{\E[b(x)(f(x)-v_S)\1( f(x)\in S)]} \\ \leq&
        \sqrt{\alpha M} + \frac{1}{K}\E[|b(x)|]\leq \sqrt{\alpha M} + \frac{\sqrt{M}}{K}.
    \end{align}
\end{proof}

\begin{remark}
For this function class $\cC$ that contains the set of signed bins $\Psi$, if we want to have a multicalibration error $\eps$, we need to pick $\alpha = \frac{\eps^2}{2M}$.
\end{remark}

\subsection{When $\cA$ can preform an AND of two intervals indicators, we get level-set multicalibration}
\label{sec:level-set-mc}
We now move away from the additive function class restriction on the augmented class $\cA$ to show that if we assume richer forms of updates are possible, we can get swap-regret guarantees with respect to an even stronger benchmark than standard multicalibration. Standard multicalibration can be thought of as a form of swap regret with respect to the target class $\cB$, in the sense that it guarantees that if you \textit{swap} the level sets of your predictor out for a predictor in $\cB$, you cannot improve in error. A stronger claim would be that you could not improve, even if you \textit{simultaneously} condition on the level sets of any specific model in $\cB$. Here, we show this form of claim, also called \textit{level set} multicalibration.

Group level-set multicalibration is defined with respect to a set of functions with discrete outputs. We start by defining it with respect to a predictor with a discrete output set $R\subseteq [0,1]$ and function class with a discrete output set $V$.
\begin{definition}\label{discrete-level-set}
    Let $\cB$ be a set of function with a discrete range $b:\cX\rightarrow V$. A predictor $f:\cX\rightarrow R$ is level-set multicalibrated with respect to $\cB$, distribution $\cD$ and error $\eps\in[0,1]$ if for all $b\in\cB$
    \begin{align*}
        \sum_{r\in R,u\in V}\Pr_{x\sim \cD}[f(x)=r,b(x)=v]\abs{\E_{(x,y)\sim\cD}[r-y|f(x)=r,b(x)=v]}\leq\eps.
    \end{align*}
\end{definition}
Level-set multicalibration has application for omnipredictors, where it allows handling the stronger types of constraints \cite{hu2023omnipredictors}, and related notion also appears in \cite{gopalan2022loss}.  

We can extend the definition to continuous hypothesis class and predictor using bins.
\begin{definition}\label{binned-level-set}
    Let $\cB$ be a set of function $b:\cX\rightarrow [0,1]$. A predictor $f:\cX\rightarrow R$ is level-set multicalibrated with respect to $\cB$ distribution $\cD$, bins $\cS'$ and error $\eps\in[0,1]$ if
    \begin{align*}
        \sum_{S'\in\cS',r\in R}\Pr_{x\sim \cD}[f(x)=r,b(x)\in S']\abs{\E_{(x,y)\sim\cD}[r-y|f(x)=r,b(x)\in S']}\leq\eps.
    \end{align*}
\end{definition}
We define the class of functions $\cA$ for this section of the paper. Unlike in the previous sections, it is not a multiplication of $b\in\cB$.

\begin{definition}\label{def:threshold-and}
    A set of hypothesis $\cA:\cX\times[0,1]\rightarrow \R$ is an AND of thresholds of a hypothesis class $\cB$ binnings $\cS,\cS'$ if for every $b\in\cB$, $S\in\cS,S'\in\cS'$ and $\eta\in\R$ there is a function $a\in\cA$ such that:
    \begin{align*}
        a(x,u)=u+\eta\1(b(x)\in S')\cdot\1(u\in S).
    \end{align*}
\end{definition}

\begin{lemma}\label{lemma:level-set}
    Let $\cS$ be a partition of $[0,1]$ into $K$ equal-sized bins and $\cS'$ be another partition of $[0,1]$ to bins. Let $\cB$ be a set of functionc $b\in\cB$ is $b:\cX\rightarrow[0,1]$.
    Let $\cA$ be an AND of thresholds of hypothesis class $\cB$ with respect to bins $\cS,\cS'$ as in Definition \ref{def:threshold-and}.
    
    Let $f \leftarrow FeatureBoost(\cA, \cD, f_0, \alpha, \tau)$ for some distribution $\cD$ and hyperparameters $\alpha, \tau \in (0,1]$. Then  $ \bar f$, the discretization of $f$ to $\cS$ is $\sqrt{\alpha} K |\cS'|+1/K$-multiacalibrated with respect to $\cB$.
\end{lemma}
\begin{proof}
   Let $f$ be the function satisfying the conditions of the lemma and $\bar f$ be the discretization of $f$ to $\cS$, as in Definition \ref{def:discretization}. 
    
    Fix any function $b\in\cB$, the level-set calibration error of $\bar f$ on $b$ is given by:
    \begin{align*}
        \sum_{S'\in\cS',r\in R}\Pr_{x\sim \cD}\left[\bar f(x)=r,b(x)\in S'\right]\abs{\E_{(x,y)\sim\cD}[r-y|\bar f(x)=v,b(x)\in S']} \\= \sum_{S\in\cS,S'\in\cS'}  \Pr_{x\sim \cD}\left[f(x)\in S,b(x)\in S'\right]\abs{\E_{(x,y)\sim\cD}[v_{S}-y|f(x)\in S,b(x)\in S']}
    \end{align*}
    For every $S\in\cS,S'\in\cS'$ we have that
    \begin{align*}
         \abs{\E_{(x,y)\sim\cD}[v_{S}-y|f(x)\in S,b(x)\in S']} =& \abs{\E_{(x,y)\sim\cD}[v_{S}-y + f(x)-f(x)|f(x)\in S,b(x)\in S']} \\\leq&
         \abs{\E[v_{S}-f(x)|f(x)\in S,b(x)\in S']} \\&+ \abs{\E[
         f(x)-y|f(x)\in S,b(x)\in S']}.
    \end{align*}
    From Theorem \ref{thm:boosting-and} on $S,S',b$, $\abs{\E[(y-f(x))\1 (f(x)\in S,b(x)\in S')]}\leq \sqrt{\alpha}$.
    For the second element in the sum, we note that $\abs{f(x)-v_S}\leq 1/K$, as this is the width of each bin in $\cS$. Therefore, we have
    \begin{align*}
        \abs{\E[(f(x)-v_S)\1 (f(x)\in S,b(x)\in S')]}\leq \frac{1}{K}\Pr[f(x)\in S,b(x)\in S']
    \end{align*}

    Combining everything together:
    \begin{align*}
        &\sum_{S\in\cS,S'\in\cS'}\Pr_{x\sim \cD}\left[f(x)\in S,b(x)\in S'\right]\abs{\E_{(x,y)\sim\cD}\left[v_{S}-y|f(x)\in S,b(x)\in S'\right]} \\&\leq \sum_{S\in\cS,S'\in\cS'}\sqrt{\alpha} + \frac{1}{K}\Pr[f(x)\in S,b(x)\in S']\\
        &\leq  K|\cS'|\sqrt{\alpha } + \frac{1}{K}.
    \end{align*}

\end{proof}

\begin{theorem}[Boosting Guarantee for an threshold AND class]\label{thm:boosting-and}
Let $\cS,\cS'$ be two partitions of $[0,1]$ into disjoint intervals.
    Let $\cA$ be an AND of thresholds of a hypothesis class $\cB,b\in\cB$ is $b:\cX\rightarrow[0,1]$ and partitions $\cS,\cS'$.  Fix $f_0: \cX \rightarrow \R$. Let $f \leftarrow FeatureBoost(\cA, \cD, f_0, \alpha, \tau)$ for some distribution $\cD$ and hyperparameters $\alpha, \tau \in (0,1]$. Then $\forall b\in \cB,S\in\cS,S'\in\cS'$,
    \[
    | \E[(y-f(x))\1(b(x)\in  S',f(x)\in S)]| \le \sqrt{\alpha}
    \]
\end{theorem}
\begin{proof}
    Let $f \leftarrow FeatureBoost(\cA, \cD, f_0, \alpha, \tau)$, and assume towards a contradiction that there exists $b\in\cB$,  $S\in\cS$ and $S'\in\cS'$ such that  $| \E[(y-f(x))\1(b(x)\in  S',f(x)\in S)]|>\sqrt{\alpha}$. 

    Let $\sigma = \text{Sign}(\E[(y-f(x))|b(x)\in  S',f(x)\in S])$, and let $f'(x) =f(x) - \sigma\eta \1(b(x)\in S',u\in S)$, where $\eta\in[0,1]$ is decided later. Then $f'\in \cA$.  
We show that $f'$ reduces the error by more than $\alpha$, contradicting the algorithm's grantee:
\begin{align*}
    L_\cD(f')<L_\cD(f) \leq& \E[(f'(x)-y)^2] - \E[(f(x)-y)^2]\\=& \E\left[\left(f(x)-\eta\sigma -y\right)^2 \1(f(x)\in S,b(x)\in S')\right] 
    - \E[(f(x)-y)^2\1(f(x)\in S,b(x)\in S')]\\=&
    \E[(-2\eta\sigma(f(x)-y) + \eta^2)\1(f(x)\in S,b(x)\in S')] \\\leq&
    \eta^2\Pr\left[f(x)\in S,b(x)\in S'\right] -2\eta\E[(f(x)-y)\1(f(x)\in S,b(x)\in S')]  \\ \leq &
    \eta^2  -2\eta\sqrt{\alpha}.
\end{align*}
We choose $\eta = \sqrt{\alpha}$ and get that
$ \eta^2  -2\eta\alpha = \alpha - 2\alpha = -\alpha$

Therefore, there exists $f'\in\cA$ such that $L_\cD(f')<L_\cD(f)-\alpha$. Let $a\in\cA^{\clip}$ be the clipped $f'$, i.e. for every $x,u$ we have $a(x,u)=\clip(f'(x,u))$. Since $y\in\{0,1\}$, the expected loss of $a$ is at most the expected loss of $f'$ and we have $L_\cD(a)\leq L_\cD(f)-\alpha$. The oracle returns a hypothesis that is $\tau$ away from the minimum-loss hypothesis. Therefore, the oracle returns a hypothesis $a'$ such that $L_\cD(a)\leq L_\cD(f)-\alpha+\tau$, and the algorithm should not have stopped and returned $f$.
\end{proof}

\section{Necessity of Discretization}
\label{sec:rounding}
One might wonder if the rounding step at the end of the algorithm is strictly necessary, or if one might be able to reason about the boosting process's guarantees for a \textit{continuous} notion of calibration while omitting the rounding step. Here, we show that this is this case: in order to get bounded continuous calibration error, rounding is necessary. 

In general, it is known that functions which are close in (e.g. $\ell_1$) distance to calibrated functions can have large continuous calibration error (defined in \ref{def:cont-mc}), as e.g. \cite{stoc23distancetocalibration} show. But the examples of such behavior are quite noisy---would such behavior happen in reality, in a boosting algorithm such as ours, thus necessitating the multicalibration guarantee to be with respect to a rounded version of the predictor as we provide in section \ref{sec:expressiveness-of-C}?

In this section we show that this is the case: our boosting algorithm may output a function which is close to a calibrated predictor in terms of $\ell_1$ distance but which has large continuous calibration error. We do so under an assumption on the weak learner oracle used: if the oracle returns the closest function $f_i\in\cF$ in each step, the unrounded output will have a large continuous calibration error. As a consequence, if we want to ensure a small continuous calibration error, the final output must be rounded to a set discretization.

\begin{definition}[Continuous expected calibration error]\label{def:cont-cal}
    The expected calibration error of a continuous $f:\cX\rightarrow[0,1]$ is
    \begin{align*}
        ECE(f)=\int_0^1 \abs{\E_{(x,y)\sim\cD}\left[\left. f(x)-y \right|f(x)=r\right]} d\mu_r,
    \end{align*}
    where $\mu_r$ is the distribution of $f(x)=r$ for $x\sim\cD$.
\end{definition}

 We can similarly define a continuous notion of multicalibration:
\begin{definition}[Multicalibration]\label{def:cont-mc}
    A continuous $f:\cX\rightarrow[0,1]$ is multicalibrated  with respect to a set of functions $\cB$
    under distribution $\cD$  with error $\eps$ if for all $b\in\cB$, 
    \begin{align*}
        \int_0^1 \abs{\E_{(x,y)\sim\cD}\left[\left. b(x)(f(x)-y) \right|f(x)=r\right]} d\mu_r \leq\eps,
    \end{align*}
    where $\mu_r$ is the distribution of $f(x)=r$ for $x\sim\cD$.
\end{definition}
We remark that even without the discretization step, the boosted function $f$ output by Algorithm \ref{algo:boosting} is close in Hamming distance to a function with small calibration error, as it is close to its discretized version $\bar f$, which has a small calibration error.\footnote{This kind of relationship is the observation underpinning the \textit{distance to calibration} literature, e.g. \cite{stoc23distancetocalibration, qiaoOnlineDistance, eshwarDistanceToCalibration}} But, since the $\ell_1$ calibration distance integrate over a continuous range, we might get a large ECE if we do not discretize, as the following example shows.

Our algorithm is a general boosting algorithm that uses as a subroutine a loss-minimization oracle with an additive slackness. We do not have any assumptions about $\cO_\cA$ other than loss-minimization, and it might preform the rounding step itself. Our counter-example shows that if the oracle $\cO_\cA$ chooses the minimum-error function at each step, then the output of the algorithm has a large calibration error.

The example is a distribution $\cD$ over $[0,1]\times\{0,1\}$, such that the probability of $y=1$ depends on $x$ but has additional high-frequency noise. Our function class cannot capture this noise, and therefore without discretization, there is a large ECE. In our example, $1/Q$ is the frequency of the noise and $\zeta$ is its strength.
\begin{example}\label{example-noisy}
    Let $\cX = [0,1]$, and $Q\in\mathbb{N}$, $\zeta\in[0,1]$.
    Let $\cD\subset\cX\times\{0,1\}$ be the distribution defined by:
    \begin{enumerate}
        \item $x$ is sampled uniformly from $[0,1]$.
        \item For $x\in\left[\frac{i}{Q},\frac{i+1}{Q}\right)$, let $\zeta'=\min\{\zeta,x,1-x\}$. For an even $i$, we set $y\sim\text{Ber}(x+\zeta')$, otherwise set $y\sim\text{Ber}(x-\zeta')$.
    \end{enumerate}
\end{example}

When proving the correctness of our example, we use the fact that out algorithm outputs a picewise linear function.
\begin{definition}\label{def:piece-lin}
    A function $f:[0,1]\rightarrow \R$ is piecewise linear with $P$ pieces if there exists a partition $\mathcal{I}$ of $[0,1]$ into $P$ disjoint intervals $\mathcal{I} = (I_1,\ldots,I_P)$, and $2P$ numbers $\left\{\nu_j,\beta_j\right\}_{j\in[P]}\in\R$ such that for all $x\in[0,1]$,
    \begin{align*}
        f(x)=\sum_{j\in[P]}\1(x\in I_j)(\nu_j x+\beta_j).
    \end{align*}
\end{definition}

\begin{claim}\label{claim:interval-number}
    Let $\cB$ be the class of one dimensional linear functions $\cB = \{\nu x + \beta | \nu,\beta\in \R\}$. Then any iterative algorithm starting from a linear function $f_0$, such that on each step $i$ preforms an update of form:
    \[ f_{t+1}(x) = \clip_{[0,1]}(f_t(x)+\eta\1(f_t(x)\in S_t)b(x)) \]
    for some $b\in\cB,S_t\subset[0,1]$ is a piecewise linear function with at most $5^{t}$ pieces
    .
\end{claim}
\begin{proof}
    We prove the claim by induction over the updates of the algorithm. We show that on every step $t$ the function $f_t$ is a piecewise linear function with $P_t\leq 5^t$ intervals, by showing that on each interval $I_j$ such that $f_t$ is linear restricting to $I_j$, we have that $f_{t+1}$ restricted to $I_j$ is a piecewise linear function with at most $5$ pieces.
    
    Basis: $f_0(x)$ is a linear function from the Claim's conditions.
    
    Step: Suppose $f_t(x)$ is a piecewise linear function with at most $5^t$ pieces. We denote by $I_1,\ldots,I_{P_t}$ the intervals of the function, where from the inductive assumption $P_t\leq 5^t$. For each piece $j$, we denote by $\nu_j x+\beta_j$ the function $f_t(x)$ restricted to $I_j$.
        
    Fix a piece $I_j$. We show that $f_{t+1}$ restricted to the interval $I_j$ is divided into at most $5$ pieces such that $f_{t+1}$ restricted to each piece is a linear function. Let $S_t$ be the interval of the update, $S = \{x| f_t(x)\in S_t, x\in I_j\}$. Since $f_t$ is linear on $I_j$, it is monotone and $S$ is a possibly empty interval in $I_j$. Let $b(x)=\nu x+\beta$ be the update function.

    Let $I_l,I_m,I_h$ be the intervals: $I_l=\{x\in I_j| \forall x'\in S, x<x'\}, I_m=I\cap S,I_h = \{x\in I_j|\forall x'\in S,x>x'\}$. We note that some intervals may be empty. 
    By definition, $I_m\subset S$ and $I_h\cap S=\emptyset,I_l\cap S=\emptyset$, and $I_j=I_l\cup I_m\cup I_h$.

    For all $x\in I_l,I_h$, We have that $f_{t+1}(x)=f_{t}(x) =\nu_j x + \beta_j$. Therefore $f_{t+1}$ restricted to $I_l,I_h$ is a linear function.

    For $x\in I_m$,  $f_{t+1}(x)= \clip_{[0,1]}(f_t(x)+\eta \nu x+\beta) = \clip_{[0,1]}((\nu_j + \eta\nu)x + \beta_j+\beta)$. The function $(\nu_j + \eta\nu)x + \beta_j+\beta$ is a linear function and therefore monotone. Clipping a monotone function to $[0,1]$ implies $I_m$ is divided into at most $3$ intervals. The function $f_{t+1}$ is the constant $1$ on one interval, the constant $0$ on another, and $(\nu_j + \eta\nu)x + \beta_j+\beta$ on the last. All functions are linear, and therefore $f_{t+1}$ on $I_m$ is a piecewise linear function with at most $3$ pieces, and on $I_j$ a piecewise linear function with at most $5$ pieces, finishing the intuction step.
\end{proof}

\begin{claim}\label{claim:calibration-error}
    Let $f:[0,1]\rightarrow[0,1]$ be a function that is $P$-piecewise linear, in intervals $I_1,\ldots,I_P$. Let $\omega,\lambda\in[0,1]$.
    If for every interval $I_j$ such that $|I_j|\geq \omega$ we have that $f(x)$ restricted to $I_j$ satisfies $f(x)=\nu_jx+\beta_j$ when $|f(x)-x|\leq\lambda$ on $I_j$, then the calibration error of $f(x)$ on the distribution of Example \ref{example-noisy} is at least 
    \begin{align*}
        &ECE(f) \geq (1-\omega P-10\lambda P)(\zeta-2\lambda) - \omega P.
    \end{align*}
\end{claim}

\begin{proof}
    We say an interval $I_j$ is large if $|I_j|\geq \omega$, and else it is small. Let $S$ be the union of all of the large intervals    \begin{align*}
        S=\cup_{j,I_j \textbf{ is large}}I_j.
    \end{align*}
    Then:
    \begin{align*}
        \Pr_x[x \in S] \geq 1-\cup_{\textbf{small } I\in\mathcal{I}}\Pr_x[x \in I]\geq 1-\omega P.
    \end{align*}
    \paragraph{Ignoring small intervals}
    In this part we lower bound the calibration error without considering small intervals. That is, we bound the calibration error of $f$ restricted to $S$.
    We remark that the density function of $f$ when we restrict $x\in S$ is $1/|S|$ the density function of $f$, since $x$ is uniform in $[0,1]$.

    We can lower bound the calibration error on large intervals, because $f$ on these intervals is approximately $f(x)=x$.
    
    Fix a large interval $I_j$, and let $v_{j,l}<v_{j,h}$ be it's boundaries (it can be that $I_j=[v_{j,l},v_{j,h}],(v_{j,l},v_{j,h})$, or other combinations). We write the function $f(x)$ restricted to $I_j$ as $f(x)=\nu_j x+\beta_j$.

    When $\lambda \ll \omega$, the function $f$ restricted to $S$ is one to one on most of the set $S$. More explicitly, for every $u\in [v_{j,l}+3\lambda,v_{j,h}-3\lambda]$ there is a single $x\in S$ such that $f(x)=u$. This is because $f$ on $I_j$ is a linear non-constant function, so it is one to one. There can't be $x'\in S\setminus I_j$ with $f(x')=u$, since then we have: $\nu' x' + \beta'=u$ and our assumption is that $|\mu'x+\beta'|\in [x-\lambda,x+\lambda]$, which cannot be in $[v_{j,l}+3\lambda,v_{j,h}-3\lambda]$.
    
    This allows us to lower bound the calibration error when restricting $f$ to $S$, using the fact that $\cD$ is uniform on $x$. Since $f$ is one to one on most of $I_j$, we can lower bound the total calibration error by lower bounding the calibration error only on the parts of $I_j$ for which $f$ is one to one on. 
    \begin{align}\label{eq:bin-ece} 
        ECE(f|_S)\geq&\sum_{j,I_j \textbf{is large}}\int_{v_{j,l}+3\lambda}^{v_{j,h}-3\lambda}\mu_{f|_S}(r)|\E[f(x)-y|f(x)=r,x\in I_j]|dr\\
        \geq&
        \sum_{j,I_j \textbf{is large}}\int_{v_{j,l}+3\lambda}^{v_{j,h}-3\lambda}\mu_{f|_S}(r)| \zeta-2\lambda|dr \\\geq&
        \sum_{j,I_j \textbf{is large}}\left(\Pr[f(x)\in [v_{j,l}+3\lambda,v_{j,h}-3\lambda]]\right)(\zeta-2\lambda)
    \end{align}
    Where we use the assumptions on $f$ and the distribution of $y$ to show that for every $x\in I_j$, $|f(x)-\E[y|x]|\geq |x(1+\lambda) + \lambda - (x+\zeta)|=|\zeta-2\lambda|$.

    For every large interval $I_j$, such that $f$ restricted to $I_j$ equals $\nu_jx+\beta_j$ for $|f(x)-x|\leq\lambda$ on $I_j$ , if $x\in [v_{j,l}+5\lambda,v_{j,h}-5\lambda]$, then $f(x)\in [v_{j,l}+3\lambda,v_{j,h}-3\lambda]$.
    This allows us to lower bound the probability appearing in the lower bound:
    \begin{align*}
        ECE(f|_S)\geq \sum_{j,I_j \textbf{is large}}\left(\Pr[x\in [v_{j,l}+5\lambda,v_{j,h}-5\lambda]]\right)(\zeta-2\lambda) \geq (|S|-10 P \lambda)(\zeta -2\lambda).
    \end{align*}
    Using the bound on $|S|$ we get that the ECE of $f|_S$ is at least $(1-PR-10 P \lambda)(\zeta -2\lambda)$.

    \paragraph{Small intervals}
    In this part we lower bound the ECE of $f$ without restricting it to the set of large intervals $S$. In this case, we cannot assume that it is one to one on any subset of the domain. Instead, we use the fact that $S$ is large and therefore $f$ defined on $[0,1]\setminus S$ cannot reduce the ECE of $f$ by much.
    
    For every $r$, let $\mu_f(r) = \mu'_f(r)+\Tilde{\mu}_f(r)$, where $\mu'_f(r)$ is the density of $f(x)=r$ when $x\in S$, and $\Tilde{\mu}_f$ is the density of $f$ when $x\notin S$.
    
    Using this notation, for all $r$ such that $\mu'_f(r)+\Tilde{\mu}_f(r)\neq 0$ we can write
    \begin{align*}
        \E[y|f(x)=r] = \frac{\mu'_f(r)}{\mu'_f(r)+\Tilde{\mu}_f(r)}\E[y|f(x)=r,x\in S] + \frac{\Tilde{\mu}_f(r)}{\mu'_f(x)+\Tilde{\mu}_f(r)}\E[y|f(x)=r,x\notin S]
    \end{align*}
    Using the triangle inequality, for every such $r$,
    \begin{align*}
        |\E[f(x)-y|f(x)=r]|\geq& \frac{\mu'_f(r)}{\mu'_f(r)+\Tilde{\mu}_f(r)}|\E[f(x)-y|f(x)=r,x\in S]|\\&-\frac{\Tilde{\mu}_f(r)}{\mu'_f(x)+\Tilde{\mu}_f(r)}|\E[f(x)-y|f(x)=r,x\notin S]|
    \end{align*}

    Combining it together, we have
    \begin{align*}
        ECE(f)\geq&\int_{0}^{1}\mu'_{f}(r)|\E[f(x)-y|f(x)=r]|dr\\
        &-\int_{0}^{1}\Tilde{\mu}_{f}(r)|\E[f(x)-y|f(x)=r]|dr
        \\\geq& (1-\omega P-10\lambda P)(\zeta-2\lambda) - \omega P.
    \end{align*}
\end{proof}

We show next that Algorithm \ref{algo:boosting}, when applied with $\cB$ being the set of linear functions, $\cC$ be a set of threshold function, and assuming the algorithm chooses the minimum-loss hypothesis at each round, has a large calibration error. Therefore, we do not exactly run Algorithm \ref{algo:boosting} but an ideal version of it. This is done as the oracle in Algorithm \ref{algo:boosting} might preform discretization inside the oracle, as explained at the beginning of this section.

\begin{lemma}
    Let $\cD$ the distribution from Example \ref{example-noisy} with $\zeta\in(0,1)$ and $Q\geq 4000P^3/\zeta^2$.
   Let $f_0(x)=0$, and $S_1=[0,1]$,$S_2,\ldots,S_T\subseteq[0,1]$ be an arbitrary collection of subsets of $[0,1]$, not necessarily disjoint.
   Let $f_T$ the function that is the output of the following iterative process for $T$ steps.
   \begin{align*}
        &(\nu_t,\beta_t) = \arg\min_{\nu,\beta\in\R}\E_{(x,y)\sim\cD}[\clip((f_{t-1}(x)+\1(f_{t-1}(x)\in S_i)(\nu x + \beta)) - y)^2]\\
       &f_{t}(x)=\clip((f_{t-1}(x)+\1(f_{t-1}(x)\in S_i)(\nu_t x + \beta_t)). 
   \end{align*}
   Then $ECE(f_T)=\Theta(\zeta)$.
\end{lemma}
\begin{proof}
    We prove the claim by bounding the $\ell_2$ distance between the current hypothesis and the function $y=x$.

    Let $f_1:\cX\rightarrow[0,1]$ be the function after a single update. Since $g(x)=x$ is a linear function, the update algorithm could have chosen $g(x)=x$ as the update.
    Let $f_t$ be the function after iteration $t$. Since the error only improves with the algorithm's run, it must be that
    \begin{align}\label{eq:dist-inequality}
        \E[(f_t(x)-y)^2]\leq\E[(f_1(x)-y)^2]\leq\E[(x-y)^2].
    \end{align}
    We add and subtract $x$
    \begin{align*}
        \E[(f_t(x)-x+x-y)^2]=\E[(f_t(x)-x)^2+(x-y)^2 - 2(x-y)(f_t(x)-x)]
    \end{align*}
    Together with (\ref{eq:dist-inequality}) we get 
    \begin{align*} 
        \E[(f_t(x)-x)^2]\leq 2 \E[(x-y)(f_t(x)-x)].
    \end{align*}
    The function $f_t(x)$ is a piecewise linear function. From Claim \ref{claim:interval-number}, $f_t$ is divided into at most $5^t$ intervals $I_1,\ldots,I_{P_t}$, and on each such interval, $f_t$ restricted to the interval is a bounded linear function, between $[0,1]$ (since we preform clipping), and therefore a bounded monotone function. For every bounded monotone function $f_t$ restricted to $I_j$, $|f_t(x)-x|\in [-1,1]$ and can change sign at most once on every $I_j$. Therefore, from the definition of the distribution $\cD$, on any interval $I_j$ we have that $\E[(x-y)(f_t(x)-x)|x\in I_j]\leq 3\zeta/Q$.
    By summing over all intervals, we have that $\E[(f_t(x)-x)^2]\leq 20 P\zeta/Q$.

    From Claim \ref{claim:interval-number}, the function $f_t$ is a $5^t$ piecewise linear function. Let $I_1,\ldots,I_P$ be the intervals on which $f_t$ is linear on.
    For each interval $j$, let $\nu_j x+\beta_j$ be the linear function on the interval. We assume that $f_t$ is already after clipping, so the function might be constant on some intervals.
    Using this notation,
    \begin{align*}
       \E[(f_t(x)-x)^2]=\sum_j \Pr[x\in I_j] \E[(\nu_j x+\beta_j-x)^2|x\in I_j]\geq\sum_j \Pr[x\in I_j]\frac{1}{4}\sup_{x\in I_j}{(v_j x+\beta_j-x)^2}
    \end{align*}
    Fix $\lambda$ that is set later. We say that $I_j$ has large error if $\sup_{x\in I_j}{|v_j x+\beta_j-x|} > \lambda$. Let $I_E=\cup_{j,I_j \text{has large error}} I_j$ 
    We combine the lower bound and upper bound on $\E[(f_t(x)-x)^2]$ and get that
    \begin{align*}
        \Pr[x\in I_E]\frac{1}{4}\lambda^2 \leq\E[(f_t(x)-x)^2]\leq \frac{20 P\zeta}{Q}
    \end{align*}
    Therefore, $\Pr[x\in I_E]\leq 80 P\zeta/(\lambda^2 Q)$.
    We set  $\lambda = \frac{\zeta}{10 P}$ and $\omega=8\zeta/ P \geq  80 P\zeta/(\lambda^2 Q)$. Therefore, for every large interval $I_j$ we have that $\sup_{x\in I_j}{|v_j x+\beta_j-x|} \leq \lambda$.

    Using Lemma \ref{claim:calibration-error}, the calibration error is lower bounded by $(1-\omega P-10\lambda P)(\zeta-2\lambda) - \omega P$.
    Since the distribution $\cD$ has $Q\geq 4000P^4/\zeta^2$, the above calibration error is at least $\frac{1}{2}\zeta$, finishing the proof.
    \end{proof}

\subsection{Expressiveness of the update class}\label{sec:expressiveness-of-C}
We show that without even a minimal expressiveness assumption on the class $\cC$, there may be a large calibration error even after rounding. The assumption we require on Section \ref{sec:interval-mc} is rather weak: namely, that $\cC$ can express threshold functions (or equivalently). To demonstrate that such assumption or a similar one is necessary, we construct a counterexample in which $\cC$ consists only of constant functions, and thus cannot express thresholds even with composition. Although this is a simple function class, it illustrates that without the ability to represent even threshold functions, calibration guarantees after rounding can fail.

\begin{example}\label{example:threshold}
   Let $\cX = [0,1]$, i.e. each $x\in \cX$ is of dimension $1$. Let $\zeta\in[0,1]$ be a parameter. We define the distribution of $y$ by $y\sim\text{Ber}(0.4 + 0.2\cdot\1(x\in [0.5,1]))$.
\end{example}

\begin{claim}
    Let $\cD$ be the distribution described on Example \ref{example:threshold}, let $\cB=\{\nu x + \beta|\nu,\beta\in\R\}$ be the set of linear functions and $\cC=\{\eta|\eta\in\R\}$ be the set of constant functions, and $f_0(x)=0$ be the constant $0$ function.

    Let $f \leftarrow FeatureBoost(\cA, \cD, f_0, \alpha, \tau)$ for $\alpha,\tau<0.01$ and let $\bar f$ be the discretization of $f$ to at least $k=10$ equal-sized bins. Then $f,\bar f$ have discretization error at least $0.01$.
\end{claim}
\begin{proof}
    We first prove that Algorithm \ref{algo:boosting}, when run on $\cD$ never outputs a hypothesis $\cA$ that needs to be clipped.

    At the first iteration, the loss minimization of the algorithm is finding $\nu,\beta$ that minimizes
    \begin{align*}
        \E_{x,y}[(\clip[\nu x +\beta] -y)^2].
    \end{align*}
    Without clipping, by calculating the derivatives the optimal $\nu,\beta$ are $\nu=0.6,\beta=0.2$, resulting in $\ell_2$ error of  $0.25$, where $\E[(0.6x+0.2 - \mu(y|x))^2]=0.01$, where $\mu(y|x)$ is the mean of $y|x$, i.e. $\mu(y|x)=0.4$ for $x\in[0,0.5]$ and $0.6$ in the rest.

    Assume towards a contradiction that Algorithm \ref{algo:boosting} outputs a different hypothesis $f_1(x)=\clip[\nu x +\beta]$, that does require clipping. Without loss of generality, assume the clipping is at $0$ (the distribution is symmetric with respect to $0,1$).
    This means that there exists $x\in[0,1]$ such that $\nu x +\beta <0$. 
    We divide into cases by the value of $\nu$. If $|\nu|<1$, then for all $x'\in I=[x+0.1,x-0.1]$ we have that $f_1(x)<0.01$, which means that $\E[(f_1(x) -\mu(y|x))^2|I\cap [0,1]]=0.03$, which already implies that the error of the capped linear function is larger by more than $0.01$ from the optimal linear function.

    In the case where $|\nu| > 1$, we have that for all $x>0.9$ or $x<0.1$ the value of $\clip[\nu x +\beta]$ is in $[0,0.1]\cup[0.9,1]$ resulting again in a large error.

    Therefore, the oracle never outputs a function that needs clipping. In this case, the set of functions the oracle outputs are only linear functions. We show next that the calibration error of every linear function on distribution $\cD$ is large.

    For every discretization to more than $4$ bins, we bound the calibration error on all bins that are fully in the range $[0,0.25]$. Since the number of bins is at least $4$, there must be at least a single bin in this range.
    For each bin $S=[i/K,(i+1)/K]$ that is fully contained in $[0,0.25]$ we have that $v_S \leq 0.25\cdot 0.6+0.2$. Therefore, the calibration error on each bin is at least $1/20$.

    If we sum over all bins that are fully in the range $[0,0.25]$, the contribution to the calibration error from these bins is at least $1/200$. 

    Since the range $[0.75,1]$ is symmetric, we have calibration error of at least $0.01$ for any discretization of at least $4$ bins.
\end{proof}

\section{Sample Complexity Analysis}
\label{sec:generalization}
We now consider an in-sample instantiation of the algorithm, and show that our guarantees generalize out-of-sample. We first show that if we assume the full augmented boosting process in $\featureboost$ as instantiated on an augmented model class $\cA$ has bounded pseudodimension then the in-sample measured accuracy and (multi)calibration guarantees hold. We then demonstrate that we can instantiate this by bounding the circuit expressivity of the updates. 

We use standard uniform convergence tools for real-valued functions from \cite{anthony1999neural}, restated in Appendix~\ref{app:textbook-results}. Our results will rely on a uniform convergence guarantee which uses a bound on covering number by pseudodimension. We will later apply this uniform convergence statement to show that our error and multicalibration guarantees generalize out of sample assuming the pseudodimension of the models generated by boosting are bounded in Sections \ref{sec:gen-err} and \ref{sec:gen-mc}. Then, in Sections \ref{sec:gen-unrolled} and \ref{sec:dt-gen}, we will show how we can instantiate this assumed pseudodimension bound.

\subsection{Generalization Preliminaries}
\label{sec:gen-prelim}

We start by formalizing pseudodimension, and giving a uniform convergence guarantee in terms of it. 

\begin{definition}[Pseudo-shattering, \cite{anthony1999neural}]
Let $\cF$ be a set of functions mapping from a domain $\cX$ to $\R$ and suppose that $X = \{x_1, \ldots, x_n\} \subseteq \cX$. Then $X$ is pseudo-shattered by $\cF$ if there exists real numbers $r_1, \ldots, r_n$ such that for each $\sigma \in \{0,1\}^n$ there is a function $f_\sigma \in \cF$ with $\sgn(f_\sigma(x_i)-r_i)=\sigma_i$ for $i \in [n]$. We say that $r = (r_1, \ldots, r_n)$ witness the shattering. 
\end{definition}

\begin{definition}[Pseudodimension, \cite{anthony1999neural}]
Let $\cF$ be a set of functions $f: \cX \rightarrow \R$. Then $\cF$ has pseudodimension $d$ if $d$ is the maximum sized subset $X$ of $\cX$ that is pseudo-shattered by $\cF$.
\end{definition}

\begin{corollary}[Uniform Convergence of Bounded Real Functions with Finite Pseudodimension]
\label{cor:uniform-conv}
Let $\cF$ be a set of functions $f: \cX \rightarrow [0,1]$ with bounded pseudodimension $d$. Fix a distribution $\cD \in \Delta(\cX \times [0,1])$ and let $\xi \in (0,1)$ and $n \in \bZ^+$. Consider any sample $D \sim \cD^n$ of $n$ points drawn independently and identically from distribution $\cD$. Then, 
\[
\Pr\left(\max_{f \in \cF} \vert L_{\cD}(f) - L_D(f) \vert \ge \xi \right) \le 4e(d+1)\left(\frac{32e}{\xi}\right)^d \exp(-\xi^2n/32).
\]
\end{corollary}
\begin{proof}
    This follows from Theorem \ref{thm:uniform-convergence-reals} and the bound in Theorem \ref{thm:covering-bound}, plugging in $\xi' = \xi/16$ into the bound. 
\end{proof}

\subsection{Risk Minimization Generalization}
\label{sec:gen-err}
With the tools from Section \ref{sec:gen-prelim} and Appendix \ref{app:textbook-results} in hand, we have almost everything we need to get a uniform convergence guarantee for the updates in our boosting process. However, we cannot simply apply uniform convergence by assuming the pseudodimension of our additive factorized hypothesis base classes $\cB$ and $\cC$ are bounded, because the hypothesis class generated at round $t$ of boosting is recursively defined in terms of a nontrivial combination and clipping of such functions. Instead, we first formalize the full boosted hypothesis class, and then provide a uniform convergence guarantee in terms of its pseudodimension.

\begin{definition}[Augmented boosted hypothesis class $\cF_t$]
\label{def:augmented-hyp-class}
Let $\cF_{\init}$ be a family of models $\finit:\cX\to \R$, and let $\cA$ be an augmented class of models $a:\cX\times\R\to\R$. We define $\cF_0 = \cF_{\init}^{\clip}$. 
For $t\ge 1$, we recursively define
\[
\cF_t = \left\{ x \mapsto a(x,f_{t-1}(x)) : f_{t-1} \in \cF_{t-1}, a \in \clipa \right\}.
\]
When $\cA$ is an additive factorized hypothesis class with respect to $\cB$ and $\cC$, we can analogously write 
\[
\cF_t = \left\{ \clip\big(x \mapsto f_{t-1}(x) + \eta b(x)c(f_{t-1}(x))\big) : f_{t-1} \in \cF_{t-1}, b \in \cB, c \in \cC \right\}.
\]
We will also consider the subset of hypotheses in $\cF_t$ which are defined for a fixed instantiation of  $f \in\cF_{t-1}$,
\[
\cF_{t,f} = \clipa[f] = \left\{x \mapsto a(x,f(x)) : a \in \clipa \right\},
\]
as well as the union of all hypotheses which could be recursively generated up through iteration $t$ of the above recursive process:
\[
\cF_{\le t} = \bigcup_{t' \in \{0,1,\ldots,t\}} \cF_{t'}. 
\]
Note that $\cF_{t,f}\subseteq \cF_t \subseteq \cF_{\le t}$.
\end{definition}

With an excess of subscripts in hand, we can now apply Corollary \ref{cor:uniform-conv} and bound the maximum deviation in loss between the sample $D$ and distribution $\cD$ for models in our unrolled hypothesis class $\cF_{\le T}$, where the algorithm runs for $T$ rounds.

\begin{remark}[Clipped generated classes]
\label{rem:additive-generated-class}
The generated classes used in the finite-sample analysis are formed from the effective oracle class used inside Algorithm~\ref{algo:boosting}.
Since Algorithm~\ref{algo:boosting} takes $\cA$ as input but calls the oracle on $\clipa$, Definition~\ref{def:augmented-hyp-class} uses $\clipa$ in both $\cF_t$ and $\cF_{t,f}$.
Thus the sample-complexity parameters below are parameters of the clipped generated predictor classes.
\end{remark}

\begin{theorem}[Uniform convergence for augmented boosted hypothesis class $\cF_t$]
\label{thm:boosting-uniform-conv}
Fix $T \in \bZ^+$, and let $d_T$ be the pseudodimension of $\cF_{\le T}$. Fix a distribution $\cD \in \Delta(\cX \times [0,1])$ and let $\xi, \delta \in (0,1)$. Then, if
\[
n \ge \tilde{O}\left(\frac{\log(1/\delta)+d_T}{\xi^2}\right),
\]
it follows that for any set $D$ of $n$ samples drawn independently and identically from $\cD$, 
\[
\Pr \left( \max_{f_t \in \cF_{\le T}} \vert L_{\cD}(f_t) - L_D(f_t) \vert \ge \xi  \right) < \delta.
\]
\end{theorem}

\begin{proof}
Note that since all functions in $\cF_0$ and $\clipa$ have range bounded by $[0,1]$, $\cF_t$ must also map to $[0,1]$, and hence $\cF_{\le T}$ will as well. Hence, applying Corollary \ref{cor:uniform-conv}, we have that 
\begin{align*}
    \Pr\bigg(\max_{f_t \in \cF_{\le T}} \vert & L_{\cD}(f_t) - L_D(f_t) \vert \ge \xi \bigg) \le 4e(d+1)\left(\frac{32e}{\xi}\right)^d \exp(-\xi^2n/32).
\end{align*}
Setting the right hand side of this equation to $\delta$ in order to bound the probability of the bad event where the in-sample and distributional error rates differ by more than $\xi$ then solving for $n$, we find that 
\begin{align*}
\delta &\ge 4e(d+1)\left(\frac{32e}{\xi}\right)^d \exp(-\xi^2n/32) \\
    \Rightarrow \quad  \log(\delta) &\ge \log\left[4e(d+1)\left(\frac{32e}{\xi}\right)^d \exp(-\xi^2n/32)\right] \\
    &= \log 4 + 1 + \log(d+1) + d \log\left(\frac{32e}{\xi}\right) - \frac{\xi^2n}{32}\\
    \Rightarrow \quad n &\ge \frac{32}{\xi^2} \left(\log(4/\delta) + 1 + \log(d+1) + d \log(32 e/\xi)\right) \\
    &= \tilde{O}\left(\frac{\log(1/\delta) + d}{\xi^2}\right).
\end{align*}
\end{proof}

This now gives us the final tool necessary to show that the models that $\featureboost$ when run in-sample generalizes. 

\begin{theorem}[Generalization of Loss Minimization of Algorithm \ref{algo:boosting}]
\label{lem:oos-per-round-error}
Fix a distribution $\cD \in \Delta(\cX \times [0,1])$ and let $\xi, \delta \in (0,1)$. Let 
\[
n \ge \tilde{O}\left(\frac{\log(1/\delta)+d_T}{\xi^2}\right)
\]
be a lower bound on the cardinality of a sample $D \sim \cD^n$. 

Fixing augmented class $\cA$, initial model $\finit$, halting parameter $\alpha$, and oracle tolerance $\tau$, consider an instantiation of Algorithm \ref{algo:boosting}, $\featureboost(\cA, D, \finit, \alpha, \tau)$. Let $f_t$ be the model generated at round $t$, and that there were $T$ total accepted updates. Assume that $d_{T+1}$, the pseudodimension of $\cF_{\le T+1}$, is finite. Then at each of these rounds $t$, if 
\begin{equation}
\label{eq:err-drop-assumption}
L_D(f_{t-1})-L_D(f_{t}) \ge \alpha - \tau,
\end{equation}
then with probability $1-\delta,$
\begin{itemize}
\item Per-round squared error drop generalizes: 
\[
L_{\cD}(f_{t-1})-L_{\cD}(f_{t}) \ge \alpha - \tau - 2\xi.
\]
\item Per-round risk minimization: 
\[
L_{\cD}(f_t) \le \min_{f \in \cF_{t,f_{t-1}}} L_{\cD}(f) + 2\xi+\tau
\]
\item For every oracle call and every $g\in\clipa[f_{t-1}]$,
\[
L_{\cD}(f_t) \le L_{\cD}(g) + 2\xi+\tau.
\]
\end{itemize}
\end{theorem}

\begin{proof}
Consider round $t$ of Algorithm \ref{algo:boosting} as run on a sample $D$ of size at least $n$ as stated above. We begin by showing that with probability at least $1-\delta$, the squared error of the model $f_t$ generated at round $t$ will improve by at least $\alpha-2\xi$ in comparison with the previous round's model $f_{t-1}$:
\begin{align*}
L_{\cD}(f_{t-1})-L_{\cD}(f_{t}) &= \bigg(L_{\cD}(f_{t-1})-L_D(f_{t-1})\bigg) - \bigg(L_{\cD}(f_{t})-L_D(f_{t})\bigg) + \bigg(L_D(f_{t-1}) - L_D(f_{t})\bigg),\\
&\ge  \bigg(L_{\cD}(f_{t-1})-L_D(f_{t-1})\bigg) - \bigg(L_{\cD}(f_t)-L_D(f_t)\bigg) + \alpha - \tau, \quad\quad \text{(By Eq. \ref{eq:err-drop-assumption})}\\
&\ge \alpha - \tau - 2\xi. \bigquad \text{(By Theorem \ref{thm:boosting-uniform-conv})}
\end{align*}
We now compare the true distributional loss of $f_t$ with that of the true risk minimizer of the class $\cF_{t,f_{t-1}}$ generated by boosting with respect to the fixed previous model $f_{t-1}$. Here, we rely on the uniform convergence guarantee from Theorem \ref{thm:boosting-uniform-conv} and the tolerance $\tau$ of each call to the risk-minimization oracle $\cO_{\clipa}.$
\begin{align*}
    L_\cD(f_t) &\le L_D(f_t)+\xi, && \text{(By theorem \ref{thm:boosting-uniform-conv})}\\
        &\le \min_{f \in \cF_{t,f_{t-1}}} L_D(f) + \tau + \xi, && \text{(By definition \ref{def:oracle})} \\
        &\le \min_{f \in \cF_{t,f_{t-1}}} L_{\cD}(f) + \tau + 2\xi. && \text{(By theorem \ref{thm:boosting-uniform-conv})}
\end{align*}
\end{proof}

\begin{corollary}[Generalized risk minimization of final predictor]
\label{cor:final-round}
Work in the setting of Theorem~\ref{lem:oos-per-round-error}, and let
$f$ be the final output of $\featureboost(\cA,D,\finit,\alpha,\tau)$.
Then, with probability at least $1-\delta$, every $g\in\clipa[f]$ satisfies
\[
L_{\cD}(f)-L_{\cD}(g)<\alpha+2\xi.
\]
\end{corollary}

\begin{proof}
Let $\widetilde f$ be the rejected predictor computed from $f$ in the last round. Since the algorithm stops,
\[
L_D(f)-L_D(\widetilde f)<\alpha-\tau.
\]
By the $\tau$-approximate oracle guarantee, for every $g\in\clipa[f]$,
\[
L_D(\widetilde f)\le L_D(g)+\tau.
\]
Combining these inequalities gives
\[
L_D(f)-L_D(g)<\alpha.
\]
Since $f,g\in\cF_{\le T+1}$, uniform convergence then gives
\[
L_{\cD}(f)-L_{\cD}(g)<\alpha+2\xi.
\]
\end{proof}

\subsection{Multicalibration Generalization}
\label{sec:gen-mc}
We now demonstrate that the multicalibration guarantees also generalize. As before, the guarantees with be with respect to a bounded version of $\cB$. We begin with the guarantee for the final model generated without rounding to discrete level sets.

\begin{theorem}[Non-rounded multicalibration generalization for additive factorized $\cA$]
\label{thm:non-rounded-mc-generalization}
Fix distribution $\cD\in\Delta(\cX\times\cY)$, and parameters $\xi,\delta,\alpha\in(0,1)$ and $\tau\in[0,\alpha)$.
Let $n$ be large enough for Theorem \ref{lem:oos-per-round-error} and Corollary \ref{cor:final-round} to hold with parameter $T$ for the generated classes of $\cA$.
Run $\featureboost(\cA,D,f_0,\alpha,\tau)$ for $D\sim\cD^n$ and $f_0\in\cF_0$, and suppose the run has at most $T$ accepted updates and halts with output $f$.
Let
\[
\gamma:=\alpha+2\xi.
\]
Assume that for every $b\in\cB$, every $c\in\cC$ with $|c(u)|\le 1$ for all $u$, every $\eta\ge0$, and every sign $\sigma\in\{-1,1\}$, the clipped signed update
\[
x\mapsto \clip\left(f(x)+\eta\sigma b(x)c(f(x))\right)
\]
belongs to $\clipa[f]$.
Then with probability at least $1-\delta$, for every $M\ge 0$, every $b\in\cB_M(\cD)$, and every such $c\in\cC$,
\[
\left|\E_{\cD}\left[b(x)c(f(x))(y-f(x))\right]\right|
\le
\sqrt{\gamma M}.
\]
\end{theorem}

\begin{proof}
We first leverage Corollary~\ref{cor:final-round} to get a guarantee on the loss of $f$ with respect to models in $\clipa$ on the distribution $\cD$. Recall that with probability at least $1-\delta$, when \featureboost is evaluated on a sample $D$ with sufficiently large size $m$, every model $g\in\clipa[f]$ satisfies
\[
L_{\cD}(f)-L_{\cD}(g)<\gamma.
\]
Condition on this event, and fix $b\in\cB_M(\cD)$, and $c\in\cC$ with $|c(u)|\le1$.
Let
\[
w(x):=b(x)c(f(x)),
\qquad
\rho:=\E_{\cD}[w(x)(y-f(x))]
\]
be the product model of $b$ and $c$ and its correlation with the residuals of the final model $f$ output by \featureboost. Since $\vert c(u) \vert \le 1$ and $b\in\cB_M(\cD)$, $\E_{\cD}[w(x)^2]\le M$.
Say that the correlation $\rho\ne0$ and let $\sigma=\sign(\rho)$ be the direction of the bias of $f$ on this weighted subset of $D$.
For any $\eta\ge0$, let
\[
g_\eta(x):=
\clip\left(f(x)+\eta\sigma w(x)\right)
\]
be the shifted then clipped model that offsets this bias. Note that by the closure assumptions, $g_\eta\in\clipa[f]$. Furthermore, because $y\in[0,1]$, clipping to $[0,1]$ cannot increase squared loss, so
\begin{align*}
L_{\cD}(g_\eta)
&\le
\E_{\cD}\left[\left(f(x)+\eta\sigma w(x)-y\right)^2\right]\\
&=
L_{\cD}(f)-2\eta|\rho|+\eta^2\E_{\cD}[w(x)^2]\\
&\le
L_{\cD}(f)-2\eta|\rho|+\eta^2M.
\end{align*}
Thus the difference in the final loss $f$ and $g_\eta$ is bounded below by $2\eta|\rho|-\eta^2M$ for every $\eta\ge0$. Choosing $\eta=\vert \rho \vert/M$ gives $\vert\rho \vert^2/M<\gamma$, and therefore
\[
\vert\rho\vert
=
\left|\E_{\cD}\left[b(x)c(f(x))(y-f(x))\right]\right|
\le \sqrt{\gamma M}.
\]
\end{proof}

We now give the final guarantee in terms of a predictor which is discretized following the run of $\featureboost$, so that we can achieve the canonical multicalibration error $K_{1,\cD}$ from Definition~\ref{def:mc}. 

\begin{lemma}[Error of equal-width discretization]
\label{lem:equal-width-discretization-error}
Let $K\in\bZ^+$, and let
\[
\cS
:=
\left\{[0,1/K),[1/K,2/K),\ldots,[(K-1)/K,1]\right\}
\]
be the partition of $[0,1]$ into $K$ equal-width bins.
If $\bar h$ is the discretization of a predictor $h:\cX\to[0,1]$ on $\cS$ from Definition~\ref{def:discretization}, then
\[
|\bar h(x)-h(x)|
\le
\frac{1}{K}
\qquad\text{for every }x\in\cX.
\]
\end{lemma}

\begin{proof}
For each bin $S\in\cS$, the discretized value $v_S$ lies in $S$: it is either the conditional mean of $h(x)$ on the event $h(x)\in S$, or an arbitrary representative chosen in $S$ when this event has zero probability.
Since every bin in $\cS$ has width $1/K$, any $x$ with $h(x)\in S$ satisfies $|\bar h(x)-h(x)|=|v_S-h(x)|\le1/K$.
\end{proof}

\begin{theorem}[Final-discretization multicalibration generalization]
\label{thm:final-discretization-mc-generalization}
Work in the setting of Theorem~\ref{thm:non-rounded-mc-generalization}.
Let $\gamma:=\alpha+2\xi$, using the parameters from that theorem.
Let $K\in\bZ^+$, let $\cS$ be the equal-width partition from Lemma~\ref{lem:equal-width-discretization-error}, and let $\bar f=\discrete(f;\cS)$.
Assume that for every sign vector $\sigma\in\{-1,1\}^{\cS}$, the signed bin function
\[
c_\sigma(u)
:=
\sum_{S\in\cS}\sigma_S\1(u\in S)
\]
belongs to $\cC$.
Then, with probability at least $1-\delta$, for every $M\ge0$ and every $b\in\cB_M(\cD)$,
\[
K_{1,\cD}(\bar f,b)
\le
\sqrt{\gamma M}
+\frac{\sqrt{M}}{K}.
\]
\end{theorem}

\subsection{Bounding the Expressivity of $\cF_{\le T}$}
\label{sec:gen-unrolled}

In the above sections, we achieved generalization assuming that the ``unrolled" hypothesis class of all functions that could be generated across $T$ rounds, $\cF_{\le T}$, has bounded pseudodimension.
We now give a reusable way to certify such a bound by counting the number of real parameters and arithmetic/comparison operations needed to evaluate the unrolled class.

\begin{definition}[Bounded arithmetic-comparison evaluation]
\label{def:bounded-arithmetic-evaluation}
Fix integers $W,k\ge1$.
We say that a real-valued predictor class $\cG$ has a $(W,k)$ bounded arithmetic-comparison evaluation if every $g\in\cG$ can be specified by at most $W$ real parameters and evaluated by a fixed algorithm which, given these parameters and a real representation of an input, computes $g$ using at most $k$ operations of the following kinds: arithmetic operations $+,-,\times,/$ on real numbers; comparisons $>,\ge,<,\le,=,\ne$ of real numbers; and conditional jumps.
A conditional jump is a branch in the fixed algorithm whose condition is the result of one of these comparisons and whose possible target instructions are fixed in advance by the algorithm.
\end{definition}

We will use the arithmetic-comparison VC dimension bound of Anthony and Bartlett, restated in Appendix~\ref{app:textbook-results} as Theorem~\ref{thm:anthony-bartlett-arithmetic-vc}.

\begin{lemma}[Pseudodimension from bounded arithmetic-comparison evaluation]
\label{lem:arith-comparison-pdim}
If a real-valued class $\cG$ has a $(W,k)$ bounded arithmetic-comparison evaluation, then
\[
\operatorname{Pdim}(\cG)
\le
O(Wk).
\]
\end{lemma}

\begin{proof}
Consider the subgraph indicator class
\[
\cH
:=
\left\{
(x,s)\mapsto \1(g(x)\ge s): g\in\cG
\right\}.
\]
For every $g\in\cG$, the corresponding indicator in $\cH$ can be computed by first evaluating $g(x)$ and then comparing the result to the additional real input $s$.
Thus $\cH$ is computed using the same $W$ real parameters and at most $k+O(1)$ arithmetic/comparison operations.
By Theorem~\ref{thm:anthony-bartlett-arithmetic-vc}, the VC dimension of this binary class is at most $O(W(k+1))$.
If a set $\{x_1,\ldots,x_m\}$ is pseudo-shattered by $\cG$ with witnesses $s_1,\ldots,s_m$, then the set $\{(x_i,s_i):i\in[m]\}$ is shattered by $\cH$.
Therefore $\operatorname{Pdim}(\cG)\le \operatorname{VCdim}(\cH)\le O(Wk)$.
\end{proof}

\begin{theorem}[Sample complexity from bounded unrolled expressivity]
\label{thm:bounded-expressivity-sample-complexity}
Fix $T\in\bZ^+$.
Suppose the generated class $\cF_{\le T+1}$ has a $(W_{T+1},k_{T+1})$ bounded arithmetic-comparison evaluation in the sense of Definition~\ref{def:bounded-arithmetic-evaluation}.
Then the conclusions of Theorems~\ref{thm:boosting-uniform-conv}, \ref{lem:oos-per-round-error}, \ref{thm:non-rounded-mc-generalization}, and \ref{thm:final-discretization-mc-generalization} hold with sample size
\[
m
\ge
\tilde{O}\left(\frac{\log(1/\delta)+W_{T+1}k_{T+1}}{\xi^2}\right),
\]
with the same parameters and assumptions as in those theorems.
\end{theorem}
\begin{proof}
By Lemma~\ref{lem:arith-comparison-pdim},
\[
d_{T+1}
=
\operatorname{Pdim}(\cF_{\le T+1})
\le
O(W_{T+1}k_{T+1})
\]
Since $\cF_{\le T}\subseteq\cF_{\le T+1}$, the same bound also upper bounds $d_T$.
Substituting these bounds into Theorems~\ref{thm:boosting-uniform-conv} and \ref{lem:oos-per-round-error} gives the displayed sample size for uniform convergence and loss generalization.
The residual-correlation and final-discretization multicalibration guarantees then follow from Theorems~\ref{thm:non-rounded-mc-generalization} and \ref{thm:final-discretization-mc-generalization}.
\end{proof}

\subsection{Generalization Guarantees for Decision Trees}
\label{sec:dt-gen}

Theorem~\ref{thm:bounded-expressivity-sample-complexity} reduces finite-sample generalization to the pseudodimension of the recursively generated boosted class. However, this is only useful if the unrolled class can be bounded for natural choices of the update family. We now give such an instantiation for clipped additive updates whose auditor class is a bounded-depth decision-tree class. We show that the sample size depends polynomially on the unrolled tree-update description length and hence the number of boosting rounds, rather than on the number of possible level sets of the final predictor.
Throughout this instantiation, $W$ counts real parameters and $k$ counts arithmetic/comparison operations in the sense of Definition~\ref{def:bounded-arithmetic-evaluation}.

\begin{theorem}[Decision-tree instantiation of the generalization bound]
\label{thm:tree-instantiation-sample-complexity}
Fix a distribution $\cD\in\Delta(\cX\times[0,1])$, parameters $\xi,\delta,\alpha\in(0,1)$ and $\tau\in[0,\alpha)$, and integers $\treedepth,K,T,q\in\bZ^+$. Take $\cX=[0,1]^q$.
Let $\cB_{\treedepth}$ be the class of depth-at-most-$\treedepth$ axis-aligned threshold decision trees on $\cX$, with leaf values in $[-1,1]$.

For the initial predictor class, take
\[
\cF_0
:=
\left\{
f_0\in\cB_{\treedepth} : f_0(x)\in[0,1]\text{ for every }x\in\cX
\right\},
\]
that is, the same tree class with leaf values restricted to $[0,1]$.
Let $\cS_K$ be the equal-width partition  of width $1/K$ from Lemma~\ref{lem:equal-width-discretization-error}.
Let $\cC_K$ be the class of bin functions
\[
\cC_K
:=
\left\{
u\mapsto
\sum_{S\in\cS_K}\theta_S\1(u\in S)
:
\theta_S\in[-1,1]\text{ for every }S\in\cS_K
\right\}.
\]
Thus every $c\in\cC_K$ is bounded by $1$, and $\cC_K$ contains the signed bin functions needed in Theorem~\ref{thm:final-discretization-mc-generalization}.
Let $\cA_{\treedepth,K}$ be the additive joint class
\[
\cA_{\treedepth,K}
:=
\left\{
(x,z)\mapsto
z+\eta b(x)c(z)
:
b\in\cB_{\treedepth},\ c\in\cC_K,\ \eta\ge0
\right\},
\]
and let $\clipa_{\treedepth,K}$ be its clipped version,
\[
\clipa_{\treedepth,K}
:=
\left\{
(x,z)\mapsto
\clip\left(z+\eta b(x)c(z)\right)
:
b\in\cB_{\treedepth},\ c\in\cC_K,\ \eta\ge0
\right\}.
\]
Run $\featureboost(\cA_{\treedepth,K},D,f_0,\alpha,\tau)$ on a sample $D\sim\cD^n$, with $f_0\in\cF_0$, $\tau<\alpha$, and suppose the run has at most $T$ accepted updates and halts with output $f$.
Let $\bar f=\discrete(f;\cS_K)$ and $\gamma:=\alpha+2\xi$.
If
\[
n
\ge
\tilde{O}\left(
\frac{
\log(1/\delta)
+
(T+2)^2(q2^{\treedepth}+K)^2
}{\xi^2}
\right),
\]
then with probability at least $1-\delta$,
\[
\left|\E_{\cD}\left[b(x)c(f(x))(y-f(x))\right]\right|
\le
\sqrt{\gamma}
\]
for every $b\in\cB_{\treedepth}$ and every $c\in\cC_K$.
Moreover, for every $b\in\cB_{\treedepth}$,
\[
K_{1,\cD}(\bar f,b)
\le
\sqrt{\gamma}
+
\frac{1}{K}.
\]
\end{theorem}

\begin{lemma}[Arithmetic evaluation of axis-aligned decision trees]
\label{lem:tree-arithmetic-evaluation}
An axis-aligned decision tree of depth at most $\treedepth$ on $\cX=[0,1]^q$ may be represented using $W_B=O(q2^{\treedepth})$ real parameters and evaluated using $k_B=O(q2^{\treedepth})$ arithmetic/comparison operations.
\end{lemma}

\begin{proof}
A depth-$\treedepth$ binary tree has at most $2^{\treedepth}-1$ internal nodes and $2^{\treedepth}$ leaves.
We upper bound each axis-aligned split by a general linear threshold over the $q$ coordinates, using $O(q)$ real parameters per internal node.
The leaf values contribute $O(2^{\treedepth})$ additional real parameters, so $W_B=O(q2^{\treedepth})$.
To evaluate the tree, the algorithm traverses a single root-to-leaf path, evaluating at most $\treedepth$ linear thresholds.
Each threshold evaluation uses $O(q)$ arithmetic/comparison operations, so the evaluation cost is $O(q\treedepth)$, which is at most $O(q2^{\treedepth})$ for $\treedepth\ge1$.
\end{proof}

\begin{lemma}[Size of the bin-function class]
\label{lem:bin-function-size}
Let $\cC_K$ be the class of bin functions
\[
\cC_K
:=
\left\{
u\mapsto
\sum_{S\in\cS_K}\theta_S\1(u\in S)
:
\theta_S\in[-1,1]\text{ for every }S\in\cS_K
\right\}.
\]
Then $\cC_K$ has a bounded arithmetic-comparison evaluation with $W_C=O(K)$ and $k_C=O(K)$.
\end{lemma}

\begin{proof}
The partition $\cS_K$ has $K$ bins, so a function $c\in\cC_K$ has $K$ cell values.
It can be evaluated by comparing $z$ to the $K-1$ interior bin thresholds and selecting the corresponding cell value.
The thresholds are fixed by the partition, while the cell values contribute $O(K)$ real parameters; the comparisons and selection use $O(K)$ arithmetic/comparison operations.
\end{proof}

\begin{lemma}[Size of one clipped additive update]
\label{lem:clipped-additive-update-size}
Suppose every $b\in\cB_{\treedepth}$ has a bounded arithmetic-comparison evaluation using at most $W_B$ real parameters and $k_B$ arithmetic/comparison operations, and every $c\in\cC_K$ has a bounded arithmetic-comparison evaluation using at most $W_C$ real parameters and $k_C$ arithmetic/comparison operations.
Then every clipped additive update
\[
(x,z)\mapsto \clip\left(z+\eta b(x)c(z)\right)
\]
has a bounded arithmetic-comparison evaluation using
\[
W_{\mathrm{upd}}\le W_B+W_C+1+O(1),
\qquad
k_{\mathrm{upd}}\le k_B+k_C+O(1).
\]
In particular, using the bounds $W_B=O(q2^{\treedepth})$, $k_B=O(q2^{\treedepth})$, $W_C=O(K)$, and $k_C=O(K)$ gives
\[
W_{\mathrm{upd}}=O(q2^{\treedepth}+K),
\qquad
k_{\mathrm{upd}}=O(q2^{\treedepth}+K).
\]
\end{lemma}

\begin{proof}
Evaluate $b(x)$ and $c(z)$ using their respective arithmetic-comparison procedures.
The step size $\eta$ is not discretized, so it contributes one additional real parameter.
The multiplication $b(x)c(z)$, multiplication by $\eta$, addition of $z$, and final clipping operation add only a constant number of parameters and arithmetic/comparison operations.
This gives the displayed generic bound, and substituting the preceding bounds for $\cB_{\treedepth}$ and $\cC_K$ gives the stated concrete bound.
\end{proof}

\begin{lemma}[Size of the unrolled generated class]
\label{lem:tree-instantiation-unrolled-size}
In the setting of Theorem~\ref{thm:tree-instantiation-sample-complexity}, the generated class $\cF_{\le T+1}$ has a bounded arithmetic-comparison evaluation using
\[
W_{T+1}
=
O\!\left(T(q2^{\treedepth}+K)\right),
\qquad
k_{T+1}
=
O\!\left(T(q2^{\treedepth}+K)\right).
\]
\end{lemma}

\begin{proof}
The initial class $\cF_0$ has the same order of description complexity as $\cB_{\treedepth}$, namely
\[
W_0=O(q2^{\treedepth}),
\qquad
k_0=O(q2^{\treedepth}).
\]
By Lemma~\ref{lem:clipped-additive-update-size}, one clipped additive update has
\[
W_{\mathrm{upd}}=O(q2^{\treedepth}+K),
\qquad
k_{\mathrm{upd}}=O(q2^{\treedepth}+K).
\]
Every predictor in $\cF_{\le T+1}$ is obtained by choosing an initial predictor and composing at most $T+1$ clipped additive updates.
Thus a direct unrolling gives
\[
W_{T+1}\le W_0+(T+1)W_{\mathrm{upd}},
\qquad
k_{T+1}\le k_0+(T+1)k_{\mathrm{upd}}.
\]
Substituting the displayed bounds for $W_0,k_0,W_{\mathrm{upd}},k_{\mathrm{upd}}$ gives the claim.
\end{proof}

\begin{lemma}[Pseudodimension of the unrolled generated class]
\label{lem:tree-instantiation-pdim}
In the setting of Theorem~\ref{thm:tree-instantiation-sample-complexity},
\[
\operatorname{Pdim}(\cF_{\le T+1})
\le
O\left(T^2(q2^{\treedepth}+K)^2\right).
\]
Consequently, the same bound also upper bounds $\operatorname{Pdim}(\cF_{\le T})$.
\end{lemma}

\begin{proof}
By Lemma~\ref{lem:tree-instantiation-unrolled-size}, $\cF_{\le T+1}$ has a bounded arithmetic-comparison evaluation using
\[
W_{T+1}
=
O\!\left(T(q2^{\treedepth}+K)\right),
\qquad
k_{T+1}
=
O\!\left(T(q2^{\treedepth}+K)\right).
\]
By Lemma~\ref{lem:arith-comparison-pdim},
\[
\operatorname{Pdim}(\cF_{\le T+1})
=
O(W_{T+1}k_{T+1}).
\]
Substituting the displayed bounds on $W_{T+1}$ and $k_{T+1}$ gives
\[
W_{T+1}k_{T+1}
=
O\left(T^2(q2^{\treedepth}+K)^2\right),
\]
where we absorb lower-order additive terms into the asymptotic notation.
Finally, $\cF_{\le T}\subseteq\cF_{\le T+1}$, so monotonicity of pseudodimension gives the same upper bound for $\operatorname{Pdim}(\cF_{\le T})$.
\end{proof}

\begin{proof}[Proof of Theorem~\ref{thm:tree-instantiation-sample-complexity}]
By Lemma~\ref{lem:tree-instantiation-pdim}, the complexity term required by Theorem~\ref{thm:bounded-expressivity-sample-complexity} is at most
\[
O\left(T^2(q2^{\treedepth}+K)^2\right).
\]
The displayed sample-size condition is exactly the one obtained by substituting this bound into Theorem~\ref{thm:bounded-expressivity-sample-complexity}.

It remains to identify the concrete consequences for calibration.
Every $b\in\cB_{\treedepth}$ is bounded by $1$, so $b\in\cB_1(\cD)$ for every distribution $\cD$.
Every $c\in\cC_K$ is also bounded by $1$.
Therefore Theorem~\ref{thm:non-rounded-mc-generalization} gives
\[
\left|\E_{\cD}\left[b(x)c(f(x))(y-f(x))\right]\right|
\le
\sqrt{\gamma}
\]
for every $b\in\cB_{\treedepth}$ and $c\in\cC_K$.
Since $\cC_K$ is closed under sign flips, for every sign $\sigma\in\{-1,1\}$ the update
\[
x\mapsto
\clip\left(f(x)+\eta\sigma b(x)c(f(x))\right)
\]
belongs to $\clipa_{\treedepth,K}[f]$.
Thus the signed-update closure assumption in Theorem~\ref{thm:non-rounded-mc-generalization} holds for this instantiation.
Finally, $\cC_K$ contains every signed bin function
\[
u\mapsto
\sum_{S\in\cS_K
}\sigma_S\1(u\in S),
\qquad
\sigma\in\{-1,1\}^{\cS_K}.
\]
Applying Theorem~\ref{thm:final-discretization-mc-generalization} with $M=1$ gives
\[
K_{1,\cD}(\bar f,b)
\le
\sqrt{\gamma}
+
\frac{1}{K}
\]
for every $b\in\cB_{\treedepth}$.
\end{proof}

\begin{remark}[Comparison with level-set boosting]
The dependence on the discretization parameter in Theorem~\ref{thm:tree-instantiation-sample-complexity} should be read as the cost of representing level-set tests inside the lifted update class. This is in contrast to the boosting procedure of \citep{lsboost}: there, once a discretization of the current predictor has been fixed, the algorithm effectively competes with corrections indexed by both a prediction bin and an auditor from the base class, for example functions of the form
\[
x\mapsto h_v(x)\1(\bar f(x)=v)
\]
for each discretized prediction value $v$.
In that work, statistical complexity therefore pays for controlling this product of level-set indicators and auditing functions, and this cost grows with the number of bins.
In the feature-augmented procedure analyzed here, the corresponding one-step competitors are represented instead as lifted functions of $(x,f(x))$, such as
\[
(x,z)\mapsto \clip\left(z+\eta b(x)c(z)\right).
\]
For the decision-tree instantiation, the class $\cC_K$ must contain all signed bin functions, so increasing the number of final discretization bins increases the update description complexity by the $O(K)$ terms in the sample bound.
Thus the guarantee does not remove the statistical cost of discretization; it moves that cost from separate per-level oracle calls to the expressivity of a single lifted update class, while the final post-processing step contributes the deterministic discretization error $1/K$ in the $K_1$ bound. However, in practice, when one actually \textit{runs} algorithms like level set boosting, the weak learning class one chooses often may in fact be closed with respect to the kinds of functions we assume here. Hence, in practice, this may correspond to sample complexity gains.
\end{remark}

\subsection{Unbounded arithmetic operations and constants may lead to unbounded dimension}
In sections \ref{sec:gen-unrolled} and \ref{sec:dt-gen}, we achieve generalization by bounding the number of total operations and ensuring that they are clipped to $[0,1]$.
In this example we show that when there is no limit on the precision of real numbers and unbounded operations, the expressiveness of the augmented function class can be unbounded. The example uses the fact that the inputs are real numbers with infinite precision in order to get an infinite VC dimension.

\begin{example}
    We will construct our hypothesis class $\cA$ as a union of two other functions classes: 
    Let $\cA=\{a\} \cup \{c_\eta : \eta\in[0,1]\}$ be the set of constant functions. I.e., 
    For every $\eta\in[0,1]$ the constant function $c_\eta(x,u)=\eta$ for all $x,u$.
    
   Additionally, we will construct the function $a:[0,1]\times[0,1]\rightarrow[0,1]$ as follows: Given $x,u\in[0,1]$, view $x,u$ as an infinite binary string, denoted $x^{(1)},x^{(2)}\ldots$ and $u^{(1)},u^{(2)}\ldots$.
    Similarly, view the string $x$ as an encoding of $(m,q)$ for some $m\in\mathbb N$ and $q_x\in [m]$.
    Given $m$, we view $u^{(1)},u^{(2)}\ldots$ as an encoding of a function $\varphi_u:[m]\rightarrow \{0,1\}$, by setting $\varphi_u(i)=u^{(i)}$.

    Then, we will define the function $a$ by:
    \begin{align*}
        a(x,u) = \varphi_u(q_x).
    \end{align*}
\end{example}

\begin{remark}
We remark that computing the function $a$ on all inputs requires an unbounded number of operations, in a standard computation model. This is because for every $d\in\mathbb{N}$, there are inputs which required reading or calculating the inputs up to $d$ digit precision.
\end{remark}

\begin{claim}
    The set $\cA$ has a finite Pseudo-dimension.
\end{claim}
\begin{proof}
    The set $\cA$ is a union of two sets: $\{a\}$ and $\{c_\eta|\eta\in\mathbb{R}\}$. The set $ \{a\}$ contains a single function and therefore has a finite Pseudo-dimention by definition. The set $\{c_\eta|\eta\in\mathbb{R}\}$ contains only constant functions, so it also has a finite Pseudo-dimension.

    The union of two finite pseudo-dimension sets also has a finite pseudo-dimension. In fact, the pseudo-dimension is at most $2$, which can be seen by considering any $x_1,x_2,x_3\in[0,1]$ and $s_1,s_2,s_3$. The function $a$ is a single function, and can therefore contribute to a single $\sigma\in\{0,1\}^3$. The set of all constant functions can express at most $4$ values of $\sigma$:  assuming without loss of generality that $s_1\geq s_2\geq s_3$, then any constant can only express the signs for $\sigma_1 \leq \sigma_2\leq\sigma_3$.
    Therefore, we cannot express all $8$ possibilities for $\sigma\in\{-1,1\}^3$. 
\end{proof}

\begin{claim}
    The class of augmenting $\cA$ with itself:
    \begin{align*}
        \cF = \{f(x,f'(x,u))|f,f'\in\cA\}
    \end{align*}
    has an infinite pseudo-dimension.
\end{claim}
\begin{proof}
    We prove the claim by showing that for every $d\in\mathbb{N}$, there is a set $x_1,\ldots,x_d\in[0,1]$ and $s_1,\ldots,s_d$ that are pseudo-shattered by $\cF$.

    Fix some $d\in\mathbb{N}$. Let $x_1,\ldots,x_d$ be such that the binary encoding of $x_i$ first encodes $d$, then $i$. Let $s_1,\ldots,s_d$ all be $0.5$.

    For every $\sigma\in\{0,1\}^d$, let $\eta_{\sigma}\in[0,1]$ be the encoding of the function $\varphi_\sigma:[d]\rightarrow\{0,1\}$, such that $\varphi(i)=\sigma_i$.
    Then, the function $f$ defined by $f_{\eta_\sigma}(x)=a(x,c_{\eta_\sigma}(x,0))$ is in the augmented class.

    The function $f_{\eta_\sigma}$ satisfies: for every $x_i$ we have that $f_{\eta_\sigma}(x_i)=\varphi_{\eta_\sigma}(i)=\sigma_i$. Since we can define such $f_{\eta_\sigma}\in\cF$ for every $\sigma\in \{-1,1\}^d$, the set $\cF$ has pseudo-dimension at least $d$. We conclude that $\cF$ has an unbounded pseudo-dimension.
\end{proof}

\section{Acknowledgments}

An LLM was used to provide notational consistency and a line-by-line audit of proof correctness, and provided suggested changes and technical fixes to formalism which were then integrated by hand. This process was iteratively performed on a near-final draft.

\bibliography{refs}
\bibliographystyle{plain}

\appendix
\section{Notation Reference}\label{app:notation}

This appendix collects the main notation used throughout the paper.

\begin{center}
\small
\renewcommand{\arraystretch}{1.25}
\begin{longtable}{p{0.18\textwidth}p{0.74\textwidth}}
\textbf{Notation} & \textbf{Meaning} \\
\hline
$\cX,\cY$ & Feature space and label space. Throughout the main analysis, labels lie in $[0,1]$, and often in $\{0,1\}$. \\
$\cD$ & Population distribution over examples $(x,y)$. \\
$D\sim\cD^n$ & Empirical sample of $m$ independent examples from $\cD$. When used inside empirical losses, $D$ also denotes the uniform distribution over this sample. \\
$n$ & Sample size. \\
$[n]$ & The set $\{1,\ldots,n\}$. \\
$L_{\cD}(f)$, $L_D(f)$ & Squared loss of a predictor $f$ under the population distribution $\cD$ and empirical sample $D$, respectively. \\
$f_0,\finit$ & Initial predictor. We use $\finit$ when emphasizing that the initial predictor is a fixed input to the algorithm. \\
$f_t$ & Predictor after $t$ accepted boosting updates. \\
$T$ & Number of accepted boosting updates under consideration. \\
$\cA$ & Augmented hypothesis class of update maps $a:\cX\times\R\to\R$. \\
$\clipa$ & Clipped version of the augmented class $\cA$ to outputs in $[0,1]$.\\
$\cF_{\le T}$ & Class of predictors obtainable from the initial class after at most $T$ accepted updates. \\
$d_T,d_{T+1}$ & Pseudodimension of the generated classes $\cF_{\le T}$ and $\cF_{\le T+1}$. \\
$\cB$ & Auditor or comparator class, usually consisting of functions $b:\cX\to\R$. \\
$\cC$ & Factor class controlling how update maps depend on the current prediction value. \\
$M$ & $L_2$ budget for auditors. For example, $\cB_M(\cD)$ denotes auditors with $\E_{\cD}[b(x)^2]\le M$. \\
$\eta$ & The update constant, typically $f_{t+1}(x)=f_t(x)+\eta a(x)$, where $a$ is the update function.\\
$\alpha$ & Threshold for squared error improvement at each round of boosting. In the exact-oracle algorithm, this is the stopping threshold; when the oracle access is approximate to tolerance $\tau$, the stopping threshold is instead $\alpha-\tau$. \\
$\tau$ & Oracle tolerance for a $\tau$-approximate squared-loss oracle. \\
$\xi$ & Uniform-convergence slack comparing empirical and population losses. \\
$\delta$ & Failure probability in high-probability generalization bounds. \\
$\gamma$ & Population-level one-step optimality slack, typically $\gamma:=\alpha+2\xi$. \\
$\eps$ & Target multiaccuracy or multicalibration error. \\
$K$ & Number of final prediction bins. \\
$\cS_K$ & Final partition of $[0,1]$ into $K$ prediction bins. \\
$S$ & A generic subset or bin, often an element of $\cS_K$. \\
$\cS'$ & A local auxiliary partition of auditor values, used only when conditioning jointly on prediction bins and auditor-value bins. \\
$\bar f$ & Final discretized predictor. \\
$\discrete(f;\cS_K)$ & Discretization of $f$ with respect to the partition $\cS_K$. \\
$\treedepth$ & Depth of an axis-aligned decision tree in the sample-complexity instantiation. \\
$q$ & Ambient feature dimension in the decision-tree instantiation, where $\cX=[0,1]^q$. \\
$\zeta$ & Noise strength in the rounding-necessity construction. \\
$Q$ & Oscillation frequency parameter in the rounding-necessity construction. \\
$P$ & Number of pieces in a piecewise-linear function in the rounding-necessity construction. \\
$\nu,\beta$ & Slope and intercept parameters for one-dimensional linear functions in the rounding-necessity construction. \\
\end{longtable}
\end{center}

\section{Comparison to previous works on the augmented boosting algorithm}
\label{sec:differences}
The work of \cite{tax2026mcgrad} primarily studies the augmented boosting algorithm empirically, showing that when used with the log loss it produces multicalibrated predictors in practice. The paper also proves, that under certain assumptions on the function classes $\cB$ and $\cC$ throughout the execution of the algorithm, the resulting predictor satisfies $\E[c(f(x))b(x)(y-f(x))]\leq\alpha$, an analog to Theorem \ref{thm:augmented-general} (only we do not have the additional assumptions). The paper does not investigate what level of expressiveness is required from $\cC$ for this guarantee to imply the standard notion of multicalibration (or stronger calibration notions), nor does it analyze theoretically the corresponding sample complexity.

The work of \cite{haimovich2026convergence} considers a finite hypothesis class $\cA=\{a_1,\ldots,a_m\}$ and studies the calibration error of a uniformly selected hypothesis. They define an error vector $\mathcal E$, whose $j$th coordinate is
\[
\mathcal E_j=\E[a_j(x,f(x))(y-\phi(f))],
\]
where $\phi$ depends on the chosen loss function. They then bound $\|\mathcal E\|$ in terms of the decrease in the loss achieved by the boosting updates.

This result provides an average-case guarantee over the hypotheses in $\cA$. In contrast, the standard notion of multicalibration is a worst-case guarantee, requiring every hypothesis in the class to have small calibration error. Therefore, the bound of \cite{haimovich2026convergence} does not by itself imply the standard multicalibration guarantee.

\section{Definitions and Theorems for Section \ref{sec:generalization}}
\label{app:textbook-results}

This appendix collects the theorems from Anthony and Bartlett's textbook ~\cite{anthony1999neural} and standard definitions that are invoked in Section~\ref{sec:generalization}. We restate them in the notation used in this paper.

\begin{definition}[Uniform Covering Number $\cN_1$ ([Paraphrase of p.148 in \cite{anthony1999neural})]
\label{def:uniform-covering-number}
Let $\cH$ be a class of functions $h: \cX \rightarrow \R$. Fix dataset $X = (x_1, \ldots, x_n) \in \cX^n$, and let 
\[
\cH_X = \{(h(x_1), \ldots, h(x_n)) : h \in \cH \} \subseteq \R^n
\]
be the collection of possible labelings induced by $\cH$ on $X$. Let 
\[
d_1(Y, Z) = \frac{1}{n} \sum_{i \in [n]} \vert y_i - z_i \vert
\]
be the normalized $\ell_1$ metric between two size $n$ samples $Y,Z \in \R^n$. We will say an $\xi$-cover of $\cH_X$ with respect to $d_1$ is a set
$V\subseteq \R^n$ such that for every $Y\in\cH_X$, there exists $Z\in V$ with $d_1(Y,Z)\le \xi$. Let $\cN_1(\xi,\cH_X)$ denote the minimum cardinality of such an $\xi$-cover. Then, we will say that the uniform $L_1$ covering number of $\cH$ is the worst (largest)-case cardinality for a sample of size $n$ of such a covering:
\[
\cN_1(\xi, \cH, n) = \sup_{X \in \cX^n} \cN_1(\xi, \cH_X).
\]
\end{definition}

\begin{theorem}[Uniform convergence for bounded real-valued functions, Theorem 17.1 of \cite{anthony1999neural}]
\label{thm:uniform-convergence-reals}
Let $\cF$ be a set of functions $f:\cX\to[0,1]$. Fix a distribution $\cD\in\Delta(\cX\times[0,1])$, let $\xi\in(0,1)$, and let $n\in\bZ^+$. For a sample $D\sim\cD^n$,
\[
\Pr\left(\max_{f\in\cF}\left|L_{\cD}(f)-L_D(f)\right|\ge\xi\right)
\le
4\cN_1(\xi/16,\cF,2n)\exp(-\xi^2n/32).
\]
\end{theorem}

\begin{theorem}[Covering numbers from pseudodimension, Theorem 18.3 of \cite{anthony1999neural}]
\label{thm:covering-bound}
Let $\cF$ be a nonempty set of functions $f:\cX\to[0,1]$ with finite pseudodimension $d$. Then
\[
\cN_1(\xi,\cF,n)
\le
e(d+1)\left(\frac{2e}{\xi}\right)^d.
\]
\end{theorem}

\begin{theorem}[VC dimension of arithmetic-comparison algorithms, Theorem 8.4 of \cite{anthony1999neural}]
\label{thm:anthony-bartlett-arithmetic-vc}
Let $h:\R^W\times\R^q\to\{0,1\}$, and let
\[
\cH
:=
\{x\mapsto h(a,x):a\in\R^W\}.
\]
Suppose $h(a,x)$ can be computed by an algorithm which, on input $(a,x)$, outputs $h(a,x)$ after at most $k$ operations of the following types: arithmetic operations $+,-,\times,/$ on real numbers; conditional jumps based on comparisons $>,\ge,<,\le,=,\ne$ of real numbers; and output of $0$ or $1$.
Then
\[
\operatorname{VCdim}(\cH)<4W(k+2).
\]
\end{theorem}

\end{document}